\documentclass{article}

\usepackage{times}
\usepackage{graphicx} 
\usepackage{subfigure}

\usepackage{url}            
\usepackage{booktabs}       
\usepackage{amsfonts}       
\usepackage{nicefrac}       
\usepackage{microtype}      

\usepackage{amsmath}
\usepackage {tikz}
\usepackage{pbox}
\usepackage{mathtools} 
\usepackage{amssymb}
\usepackage{wrapfig}

\usepackage{algorithm}
\usepackage{algorithmic}

\usepackage{hyperref}

\usepackage[accepted]{icml2019}

\usepackage{enumitem}
\setlist{nolistsep}

\usepackage{amsthm}	
\newtheorem{defn}{Definition}[]
\newtheorem{thm}{Theorem}[]

\newtheorem{lemma}{Lemma}[]

\newtheorem{ass}{Assumption}[]

\usepackage{theoremref} 
\usepackage{amsxtra} 
\usepackage[font=small]{caption}

\newcommand{\pib}{{\pi_b}}
\newcommand{\pitheta}{{\pi_\btheta}}

\newcommand{\pieval}{{\pi_e}}
\newcommand{\btheta}{{\boldsymbol\theta}}

\newcommand{\var}{{\operatorname{Var}}}
\newcommand{\mse}{{\operatorname{MSE}}}
\newcommand{\IS}{\operatorname{IS}}
\newcommand{\RIS}{\operatorname{RIS}}
\newcommand{\OIS}{\operatorname{OIS}}
\newcommand{\REG}{\operatorname{REG}}
\newcommand{\WIS}{\operatorname{WIS}}
\newcommand{\WDR}{\operatorname{WDR}}
\newcommand{\PDIS}{\operatorname{PDIS}}
\newcommand{\DR}{\operatorname{DR}}

\newcommand{\Dtrain}{\mathcal{D}_\mathtt{train}}

\newcommand{\pidata}{{\pi_\mathcal{D}}}

\def\Exp#1#2{{\mathbf{E}_#1\left[#2\right]}}
\newcommand{\loss}{{\mathcal{L}}}
\newcommand{\data}{{\mathcal{D}}}

\DeclareMathOperator*{\argmax}{\arg\!\max}

\newtheorem{corollary}{Corollary}

\newtheorem{proposition}{Proposition}

\icmltitlerunning{Importance Sampling Policy Evaluation with an Estimated Behavior Policy}

\begin{document}

\twocolumn[
\icmltitle{Importance Sampling Policy Evaluation with an Estimated Behavior Policy}

\icmlsetsymbol{equal}{*}

\begin{icmlauthorlist}
\icmlauthor{Josiah P. Hanna}{uta}
\icmlauthor{Scott Niekum}{uta}
\icmlauthor{Peter Stone}{uta}
\end{icmlauthorlist}

\icmlaffiliation{uta}{The University of Texas at Austin, Austin, Texas, USA}
\icmlcorrespondingauthor{Josiah P. Hanna}{jphanna@cs.utexas.edu}

\icmlkeywords{off-policy evaluation, importance-sampling}

\vskip 0.3in
]

\printAffiliationsAndNotice{}

\begin{abstract}

We consider the problem of off-policy evaluation in Markov decision processes. Off-policy evaluation is the task of evaluating the expected return of one policy with data generated by a different, \textit{behavior} policy. Importance sampling is a technique for off-policy evaluation that re-weights off-policy returns to account for differences in the likelihood of the returns between the two policies. In this paper, we study importance sampling with an estimated behavior policy where the behavior policy estimate comes from the same set of data used to compute the importance sampling estimate. We find that this estimator often lowers the mean squared error of off-policy evaluation compared to importance sampling with the true behavior policy or using a behavior policy that is estimated from a separate data set. Intuitively, estimating the behavior policy in this way corrects for error due to sampling in the action-space. Our empirical results also extend to other popular variants of importance sampling and show that estimating a non-Markovian behavior policy can further lower large-sample mean squared error even when the true behavior policy is Markovian.

\end{abstract}

\section{Introduction}

%
%
%
%

Sequential decision-making tasks, such as a robot manipulating objects or an autonomous vehicle deciding when to change lanes, are ubiquitous in artificial intelligence. 
For these tasks, \textit{reinforcement learning} (RL) algorithms provide a promising alternative to hand-coded skills, allowing sequential decision-making agents to acquire policies autonomously given only a reward function measuring task performance \cite{sutton1998reinforcement}.
When applying RL to real world problems, an important problem that often comes up is \textit{policy evaluation}.
In policy evaluation, the goal is to determine the expected return -- sum of rewards -- that an \textit{evaluation policy}, $\pieval$, will obtain when deployed on the task of interest.

In \emph{off-policy} policy evaluation, we are given data (in the form of state-action-reward trajectories) generated by a second \emph{behavior policy}, $\pib$.
We then use these trajectories to evaluate $\pieval$.
Accurate off-policy policy evaluation is especially important when we want to know the value of a policy before it is deployed in the real world or have many policies to evaluate and want to avoid running each one individually.
\textit{Importance sampling} addresses this problem by re-weighting returns generated by $\pib$ such that they are unbiased estimates of $\pieval$ \cite{precup2000eligibility}.
While the basic importance sampling estimator is often noted in the literature to suffer from high variance, more recent importance sampling estimators have lowered this variance \cite{thomas2016data-efficient,jiang2016doubly}.
Regardless of additional variance reduction techniques, all importance sampling variants compute the likelihood ratio $\frac{\pieval(a|s)}{\pib(a|s)}$ for all state-action pairs in the off-policy data.

In this paper, we propose to replace $\pib(a|s)$ with its empirical estimate -- that is, we replace the probability of sampling an action in a particular state with the frequency at which that action actually occurred in that state in the data.
It is natural to assume that such an estimator will yield worse performance since it replaces a known quantity with an estimated quantity.
However, research in the multi-armed bandit \cite{li2015toward,narita2019efficient}, causal inference \cite{hirano2003efficient,rosenbaum1987model}, and Monte Carlo integration \cite{henmi2007importance,delyon2016integral} literature has demonstrated that estimating the behavior policy can \textit{improve} the mean squared error of importance sampling policy evaluation.
Motivated by these results, we study the performance of such methods for policy evaluation in full Markov decision processes.

Specifically, we study a family of estimators that, given a dataset, $\mathcal{D}$, of trajectories, use $\mathcal{D}$ both to estimate the behavior policy and then to compute the importance sampling estimate.
Though related to methods in the statistics literature, the so-called regression importance sampling methods are specific to Markov decision processes where actions taken at one time-step influence the states and rewards at future time-steps.
We show empirically that regression importance sampling \textit{lowers} the mean squared error of importance sampling off-policy evaluation in both discrete and continuous action spaces.
Though our study is primarily empirical, we present theoretical results that, when the policy class of the estimated behavior policy is specified correctly, regression importance sampling is consistent and has asymptotically lower variance than using the true behavior policy for importance sampling.
To the best of our knowledge, we are the first to study this method for policy evaluation in Markov decision processes.

\section{Preliminaries}

This section formalizes our problem and introduces importance sampling off-policy evaluation.

\subsection{Notation}

We assume the environment is a finite horizon, episodic \emph{Markov decision process} with state space $\mathcal{S}$, action space $\mathcal{A}$, transition probabilities, $P$,  reward function $R$, horizon $L$, discount factor $\gamma$, and initial state distribution $d_0$ \cite{puterman2014markov}. 
A \textit{Markovian} policy, $\pi$, is a function mapping the current state to a probability distribution over actions; a policy is \textit{non-Markovian} if its action distribution is conditioned on past states or actions.
For simplicity, we assume that $\mathcal{S}$ and $\mathcal{A},$ are finite and that probability distributions are probability mass functions.\footnote{
Unless otherwise noted, all results and discussion apply equally to the discrete and continuous setting.
} 
Let $H\coloneqq(S_0,A_0,R_0,S_1,\dotsc,S_{L-1},A_{L-1},R_{L-1})$ be a \textit{trajectory}, $g(H)\coloneqq \sum_{t=0}^{L-1} \gamma^t R_t$ be the \textit{discounted return} of trajectory $H$, and  $v(\pi)\coloneqq \mathbf{E}[g(H) | H \sim \pi]$ be the expected discounted return when the policy $\pi$ is used starting from state $S_0$ sampled from the initial state distribution. 
We assume that the transition and reward functions are unknown and that the episode length, $L$, is a finite constant.

In off-policy policy evaluation, we are given a fixed \emph{evaluation policy}, $\pieval$, and a data set of $m$ trajectories and the policies that generated them: $\mathcal{D} \coloneqq \{H_i, \pib^{(i)}\}_{i=1}^{m}$ where $H_i \sim \pib^{(i)}$.
%
%
We assume that $\forall \{H_i, \pib^{(i)}\} \in \mathcal{D}$, $\pib^{(i)}$ is Markovian  i.e., actions in $\mathcal{D}$ are independent of past states and actions given the immediate preceding state.
Our goal is to design an off-policy estimator, $\operatorname{OPE}$, that takes $\mathcal{D}$ and estimates $v(\pieval)$ with minimal mean squared error (MSE).
Formally, we wish to minimize $\mathbf{E}_{\mathcal{D}}[(\operatorname{OPE}(\pieval, \mathcal{D}) - v(\pieval))^2]$.

%
%

\subsection{Importance Sampling}

\textit{Importance Sampling} (IS) is a method for reweighting returns generated by a \emph{behavior} policy, $\pib$, such that they are unbiased returns from the \emph{evaluation} policy.
Given a set of $m$ trajectories and the policy that generated each trajectory, the IS off-policy estimate of $v(\pieval)$ is:
\begin{equation}\label{eq:OIS}
\operatorname{IS}(\pieval, \mathcal{D}) \coloneqq \frac{1}{m}\sum_{i=1}^{m} g(H^{(i)}) \prod_{t=0}^{L-1} \frac{\pi_e(A_t^{(i)} | S_t^{(i)})}{\pib^{(i)}(A_t^{(i)} | S_t^{(i)})}.
\end{equation}
We refer to (\ref{eq:OIS}) -- that uses the true behavior policy -- as the ordinary importance sampling ($\OIS$) estimator and refer to $\frac{\pieval(A | S)}{\pib(A | S)}$ as the OIS weight for action $A$ in state $S$.

The importance sampling estimator with $\OIS$ weights can be understood as a Monte Carlo estimate of $v(\pieval)$ with a correction for the distribution shift caused by sampling trajectories from $\pib$ instead of $\pieval$.
As more data is obtained, the empirical frequency of any trajectory approaches the expected frequency under $\pib$ and then the $\OIS$ weight corrects the weighting of each trajectory to reflect the expected frequency under $\pieval$.

%
%

\section{Sampling Error in Importance Sampling\label{sec:objections}}

The ordinary importance sampling estimator (\ref{eq:OIS}) is known to have high variance.
A number of importance sampling variants have been proposed to address this problem, however, all such variants use the $\OIS$ weight.
The common reliance on $\OIS$ weights suggest that an implicit assumption in the RL community is that $\OIS$ weights lead to the most accurate estimate.
%
Hence, when an application requires estimating an unknown $\pib$ in order to compute importance weights, the application is implicitly assumed to only be approximating the desired weights.

However, $\OIS$ weights themselves are sub-optimal in at least one respect: the weight of each trajectory in the $\OIS$ estimate is inaccurate unless we happen to observe each trajectory according to its true probability.
%
When the empirical frequency of any trajectory is unequal to its expected frequency under $\pib$, the $\OIS$ estimator puts either too much or too little weight on the trajectory.
We refer to error due to some trajectories being either over- or under-represented in $\mathcal{D}$ as \textit{sampling error}.
%
%
%
Sampling error may be unavoidable when we desire an unbiased estimate of $v(\pieval)$. 
However, correcting for it by properly weighting trajectories will, in principle, give us a lower mean squared error estimate. 

The problem of sampling error is related to a Bayesian objection to Monte Carlo integration techniques: $\OIS$ ignores information about the closeness of trajectories in $\mathcal{D}$ \cite{hagan1987monte,ghahramani2003bayesian}.
This objection is easiest to understand in deterministic and discrete environments though it also holds for stochastic and continuous environments.
In a deterministic environment, additional samples of any trajectory, $h$, provide no new information about $v(\pieval)$ since only a single sample of $h$ is required to know $g(h)$.
However, the more times a particular trajectory appears, the more weight it receives in an $\OIS$ estimate even though the correct weighting of $g(h)$, $\Pr(h | \pieval)$, is known since $\pieval$ is known.
In stochastic environments, it is reasonable to give more weight to recurring trajectories since the recurrence provides additional information about the unknown state-transition and reward probabilities.
However, ordinary importance sampling also relies on sampling to approximate the known policy probabilities.

Finally, we note that the problem of sampling error applies to any variant of importance sampling using $\OIS$ weights, e.g., weighted importance sampling \cite{precup2000eligibility}, per-decision importance sampling \cite{precup2000eligibility}, the doubly robust estimator \cite{jiang2016doubly,thomas2016data-efficient}, and the MAGIC estimator \cite{thomas2016data-efficient}.
Sampling error is also a problem for on-policy Monte Carlo policy evaluation since Monte Carlo is the special case of $\OIS$ when the behavior policy is the same as the evaluation policy.

\section{Regression Importance Sampling} 
\label{sec:RIS}

In this section we introduce the primary focus of our work: a family of estimators called regression importance sampling ($\RIS$) estimators that correct for sampling error in $\mathcal{D}$ by importance sampling with an estimated behavior policy.
%
%
The motivation for this approach is that, though $\mathcal{D}$ was sampled with $\pib$, the trajectories in $\mathcal{D}$ may appear as if they had been generated by a different policy, $\pidata$.
For example, if $\pib$ would choose between two actions with equal probability in a particular state, the data might show that one action was selected more often than the other in that state.
Thus instead of using $\OIS$ to correct from $\pib$ to $\pieval$, we introduce $\RIS$ that corrects from $\pidata$ to $\pieval$.

We assume that, in addition to $\mathcal{D}$, we are given a policy class -- a set of policies -- $\Pi^n$ where each $\pi \in \Pi^n$ is a distribution over actions conditioned on an $n$-step state-action history: $\pi: \mathcal{S}^{n+1} \times \mathcal{A}^{n} \rightarrow [0, 1]$.
Let $H_{t-n:t}$ be the trajectory segment: $S_{t-n}, A_{t-n}, ... S_{t-1}, A_{t-1}, S_t$ where if $t - n < 0$ then $H_{t-n:t}$ denotes the beginning of the trajectory until step $t$.
The $\RIS(n)$ estimator first estimates the maximum likelihood behavior policy in $\Pi^n$ given $\mathcal{D}$:
\begin{equation}\label{eq:RIS}
\pidata^{(n)} := \argmax_{\pi \in \Pi^n} \sum_{H \in \mathcal{D}} \sum_{t=0}^{L-1} \log \pi(a | H_{t-n:t}). 
\end{equation}
The $\RIS(n)$ estimate is then the importance sampling estimate with $\pidata^{(n)}$ replacing $\pib$:
\[ \RIS(n)(\pieval, \mathcal{D}):= \frac{1}{m} \sum_{i=1}^{m} g(H_i) \prod_{t=0}^{L-1} \frac{\pieval(A_t | S_t)}{\pidata^{(n)}(A_t | H_{t-n:t})} \]
Analogously to $\OIS$, we refer to $\frac{\pieval(A_t | S_t)}{\pidata^{(n)}(S_t | H_{t-n:t})}$ as the $\RIS(n)$ weight for action $A_t$, state $S_t$, and trajectory segment $H_{t-n:t}$.
Note that the $\RIS(n)$ weights are always well-defined since $\pidata^{(n)}$ never places zero probability mass on any action that occurred in $\mathcal{D}$.


%
%
%

\subsection{Correcting Importance Sampling Sampling Error}\label{sec:example}

We now present an example illustrating how $\RIS$ corrects for sampling error in off-policy data.

Consider a deterministic MDP with finite $|\mathcal{S}|$ and $|\mathcal{A}|$. Let $\mathcal{H}$ be the (finite) set of possible trajectories under $\pib$ and suppose that our observed data, $\mathcal{D}$, contains at least one of each $h \in \mathcal{H}$.
In this setting, the maximum likelihood behavior policy can be computed with count-based estimates.
We define $c(h_{i:j})$ as the number of times that trajectory segment $h_{i:j}$ appears during any trajectory in $\mathcal{D}$. 
Similarly, we define $c(h_{i:j},a)$ as the number of times that action $a$ is observed following trajectory segment $h_{i:j}$ during any trajectory in $\mathcal{D}$.
$\RIS(n)$ estimates the behavior policy as: \[\pidata(a | h_{i-n:i}) := \frac{c(h_{i-n:i},a)}{c(h_{i-n:i})}.\]



%
Observe that both $\OIS$ and all variants of $\RIS$ can be written in one of two forms:
\[ \underbrace{\frac{1}{m} \sum_{i=1}^m \frac{w_\pieval(h_i)}{w_\pi(h_i)} g(h_i)}_{(i)} = \underbrace{\sum_{h \in \mathcal{H}} \frac{c(h)}{m} \frac{w_\pieval(h)}{w_\pi(h)} g(h)}_{(ii)} \]
where $w_\pi(h)$ = $\prod_{t=0}^{L-1}\pi(a_t|s_t)$ and for $\OIS$ $\pi \coloneqq \pib$ and for $\RIS(n)$ $\pi \coloneqq \pidata^{(n)}$ as defined in Equation (\ref{eq:RIS}).

If we had sampled trajectories using $\pidata^{(L-1)}$ instead of $\pib$, in our deterministic environment, the probability of each trajectory would be $\Pr(H | \pidata^{(L-1)}) = \frac{c(H)}{m}$.
Thus Form (ii) can be written as:
\[
\mathbf{E}\left[\frac{w_\pieval(H)}{w_\pi(H)} g(H) | H \sim \pidata^{(L - 1)}\right].
\]

To emphasize what we have shown so far: $\OIS$ and $\RIS$ are both sample-average estimators whose estimates can be written as exact expectations.
However, this exact expectation is under the distribution that trajectories were observed and \textit{not} the distribution of trajectories under $\pib$.

Consider choosing $w_\pi:=w_\pidata^{(L-1)}$ as $\RIS(L-1)$ does.
This choice results in (ii) being exactly equal to $v(\pieval)$\footnote{This statement follows from the importance sampling identity: $\mathbf{E}[\frac{\Pr(H | \pieval)}{\Pr(H | \pi)} g(h) | H \sim \pi] = \mathbf{E}[g(H) | H \sim \pieval] = v(\pieval)$ and the fact that we have assumed a deterministic environment.}
On the other hand, choosing $w_\pi:=w_\pib$ will \textit{not} return $v(\pieval)$ unless we happen to observe each trajectory at its expected frequency (i.e., $\pidata^{(L-1)} = \pib$).


Choosing $w_\pi$  to be $w_{\pidata^{(n)}}$ for $n < L - 1$ also does \textit{not} result in $v(\pieval)$ being returned in this example.
This observation is surprising because even though we know that the true $\Pr(h | \pib) = \prod_{t=0}^{L-1} \pib(a_t | s_t)$, it does not follow that the estimated probability of a trajectory is equal to the product of the estimated Markovian action probabilities, i.e., that $\frac{c(h)}{m} = \prod_{t=0}^{L-1} \pidata^{(0)}(a_t | s_t)$.
With a finite number of samples, the data may have higher likelihood under a non-Markovian behavior policy -- possibly even a policy that conditions on all past states and actions.
Thus, to fully correct for sampling error, we must importance sample with an estimated non-Markovian behavior policy.
%
However, $w_{\pidata^{(n)}}$ with $n < L - 1$ still provides a better sampling error correction than $w_\pib$ since any $\pidata^{(n)}$ will reflect the statistics of $\mathcal{D}$ while $\pib$ does not. 
This statement is supported by our empirical results comparing $\RIS(0)$ to $\OIS$ and a theoretical result we present in the following section that states that $\RIS(n)$ has lower asymptotic variance than $\OIS$ for all $n$.

Before concluding this section, we discuss two limitations of the presented example -- these limitations are \textit{not} present in our theoretical or empirical results.
First, the example lacks stochasticity in the rewards and transitions.
In stochastic environments, sampling error arises from sampling states, actions, and rewards while in deterministic environments, sampling error only arises from sampling actions.
Neither $\RIS$ nor $\OIS$ can correct for state and reward sampling error since such a correction requires knowledge of what the true state and reward frequencies are and these quantities are typically unknown in the MDP policy evaluation setting.

Second, we assumed that $\mathcal{D}$ contains at least one of each trajectory possible under $\pib$.
If a trajectory is absent from $\mathcal{D}$ then $\RIS(L-1)$ has non-zero bias.
%
%
Theoretical analysis of this bias for both $\RIS(L-1)$ and other $\RIS$ variants is an open question for future analysis.

\subsection{Theoretical Properties of RIS}\label{sec:RIS:theory}


%
%
Here, we briefly summarize new theoretical results (full proofs appear in the appendices) as well as a connection to prior work from the multi-armed bandit literature:
\begin{itemize}
\item \textbf{Proposition 1:} For all $n$, $\RIS(n)$ is a biased estimator, however, it is consistent provided $\pib \in \Pi^n$ (see Appendix A for a full proof).
\item \textbf{Corollary 1:} For all $n$, if $\pib \in \Pi^n$ then $\RIS$ has asymptotic variance at most that of $\OIS$. This result is a corollary to a result by Henmi et al.\ \yrcite{henmi2007importance} for general Monte Carlo integration (see Appendix B for a full proof). We highlight that the derivation of this result includes some $o(n)$ and $o_p(1)$ terms that may be large for small sample sizes; the lower variance is asymptotic and we leave analysis of the finite-sample variance of $\RIS$ to future work.
\item \textbf{Connection to REG:} For finite MDPs, Li et al.\ \yrcite{li2015toward} introduce the \textit{regression} ($\REG$) estimator and show it has asymptotic lower minimax MSE than $\OIS$ provided the estimator has full knowledge of the environment's transition probabilities. With this knowledge $\REG$ can correct for sampling error in both the actions and state transitions. $\RIS(L-1)$ is an approximation to $\REG$ that only corrects for sampling error in the actions. The derivation of the connection between $\REG$ and $\RIS(L-1)$ is given in Appendix C.
\end{itemize}

We also note that prior theoretical analysis of importance sampling with an estimated behavior policy has made the assumption that $\pidata$ is estimated independently of $\mathcal{D}$ \cite{dudik2011doubly,farajtabar2018more}.
This assumption simplifies the theoretical analysis but makes it inapplicable to regression importance sampling.

\subsection{RIS with Function Approximation}

The example in Section \ref{sec:example} presented $\RIS$ with count-based estimation of $\pidata$.
In many practical settings, count-based estimation of $\pidata$ is intractable and we must rely on function approximation.
For example, in our final experiments we learn $\pidata$ as a Gaussian distribution over actions with the mean given by a neural network.
Two practical concerns arise when using function approximation for $\RIS$: avoiding over-fitting and selecting the function approximator.

$\RIS$ uses all of the data available for off-policy evaluation to both estimate $\pidata$ and compute the off-policy estimate of $v(\pieval)$.
Unfortunately, the $\RIS$ estimate may suffer from high variance if the function approximator is too expressive and $\pidata$ is over-fit to our data.
Additionally, if the policy class of $\pib$ is unknown, it may be unclear what is the right function approximation representation for $\pidata$.
A practical solution is to use a validation set -- distinct from $\mathcal{D}$ -- to select an appropriate policy class and appropriate regularization criteria for $\RIS$.
This solution is a small departure from the previous definition of $\RIS$ as selecting $\pidata$ to maximize the log likelihood on $\mathcal{D}$.
Rather, we select $\pidata$ to maximize the log likelihood on $\mathcal{D}$ while avoiding over-fitting.
This approach represents a trade-off between robust empirical performance and potentially better but more sensitive estimation with $\RIS$.

\section{Empirical Results}

We present an empirical study of the $\RIS$ estimator across several policy evaluation tasks.
Our experiments are designed to answer the following questions:
\begin{enumerate}
\setlength{\itemsep}{0.5pt}
\item What is the empirical effect of replacing $\OIS$ weights with $\RIS$ weights in sequential decision making tasks?
\item How important is using $\mathcal{D}$ to both estimate the behavior policy and compute the importance sampling estimate?
\item How does the choice of $n$ affect the $\mse$ of $\RIS(n)$?
\end{enumerate}

With non-linear function approximation, our results suggest that the standard supervised learning approach of model selection using hold-out validation loss may be sub-optimal for the regression importance sampling estimator.
Thus, we also investigate the question:

\begin{enumerate}
\setcounter{enumi}{3}
\item Does minimizing hold-out validation loss set yield the minimal MSE regression importance sampling estimator when estimating $\pidata$ with gradient descent and neural network function approximation?
\end{enumerate}

\subsection{Empirical Set-up}

We run policy evaluation experiments in several domains.
We provide a short description of each domain here; a complete description and additional experimental details are given in Appendix E.\footnote{Code is provided at \url{https://github.com/LARG/regression-importance-sampling}.}

\begin{itemize}
\setlength{\itemsep}{0pt}
\item \textbf{Gridworld:} This domain is a $4 \times 4$ Gridworld used in prior off-policy evaluation research \cite{thomas2016data-efficient,hanna2017data-efficient}.
$\RIS$ uses count-based estimation of $\pib$.
This domain allows us to study $\RIS$ separately from questions of function approximation.
\item \textbf{SinglePath:} See Figure \ref{fig:singlepath} for a description. This domain is small enough to allow implementations of $\RIS(L-1)$ and the $\REG$ method from Li et al.\ \yrcite{li2015toward}.
All $\RIS$ methods use count-based estimation of $\pib$.
\item \textbf{Linear Dynamical System:} This domain is a point-mass agent moving towards a goal in a two dimensional world by setting $x$ and $y$ acceleration. Policies are linear in a second order polynomial transform of the state features. We estimate $\pidata$ with least squares.
\item \textbf{Simulated Robotics:} We also use two continuous control tasks from the OpenAI gym: Hopper and HalfCheetah.\footnote{For these tasks we use the Roboschool versions: \url{https://github.com/openai/roboschool}}
In each task, we use neural network policies with $2$ layers of $64$ $\tanh$ hidden units each for $\pieval$ and $\pib$.
\end{itemize}

\begin{figure}
\centering
\begin{tikzpicture}[-latex ,auto ,node distance =1.75 cm and 2.75cm ,
semithick ,
state/.style ={ circle ,top color =white , bottom color = white,
draw,black , text=black , minimum width =1 cm}, inf/.style = {}]
\node[state] (start) at (0.0, 0.0){$s_0$};
\node[state] (next) at (2.0, 0.0) {$s_1$};
\node[inf] (dots) at (4.0, 0.0) {...};
\node[state] (end) at (6.0, 0.0) {$s_5$};
\path (start) edge [bend left] node[above=0.1 cm,align=left] {$a_0$} (next);
\path (start) edge [bend right] node[below=0.1 cm,align=left] {$a_1$} (next);
\path (next) edge [bend left] node[above=0.1 cm,align=left] {$a_0$} (dots);
\path (next) edge [bend right] node[below=0.1 cm,align=left] {$a_1$} (dots);
\path (dots) edge [bend left] node[above=0.1 cm,align=left] {$a_0$} (end);
\path (dots) edge [bend right] node[below=0.1 cm,align=left] {$a_1$} (end);
\end{tikzpicture}
\caption{The SinglePath MDP. This environment has $5$ states, $2$ actions, and $L=5$. The agent begins in state $0$ and both actions either take the agent from state $n$ to state $n+1$ or cause the agent to remain in state $n$. \textbf{Not shown:} If the agent takes action $a_1$ it remains in its current state with probability $0.5$.}
\label{fig:singlepath}
\end{figure}
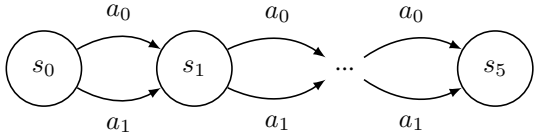

\subsection{Empirical Results}

We now present our empirical results.
Except where specified otherwise, $\RIS$ refers to $\RIS(0)$.

\paragraph{Finite MDP Policy Evaluation}

Our first experiment compares several importance sampling variants implemented with both $\RIS$ weights and $\OIS$ weights.
Specifically, we use the basic IS method described in Section 2, the \textit{weighted} IS estimator \cite{precup2000eligibility}, and the \textit{weighted doubly robust} estimator \cite{thomas2016data-efficient}.


%
%
Figure \ref{fig:gridworld} shows the $\mse$ of the evaluated methods averaged over $100$ trials.
The results show that using $\RIS$ weights improves all IS variants relative to $\OIS$ weights.\footnote{We also implemented and evaluated \emph{per-decision} importance sampling and the ordinary \emph{doubly robust} estimator and saw similar results. However we defer these results to Appendix F for clarity.}

We also evaluate alternative data sources for estimating $\pidata$ in order to establish the importance of using $\mathcal{D}$ to both estimate $\pidata$ and compute the value estimate.
Specifically, we consider:
\begin{enumerate}
\setlength{\itemsep}{0pt}
\item \textbf{Independent Estimate}: In addition to $\mathcal{D}$, this method has access to an additional set, $\Dtrain$. The behavior policy is estimated with $\Dtrain$ and the policy value estimate is computed with $\mathcal{D}$. Since $(s,a)$ pairs in $\mathcal{D}$ may be absent from $\Dtrain$ we use Laplace smoothing to ensure that the importance weights are well-defined.
\item \textbf{Extra-data Estimate}: This baseline is the same as \textbf{Independent Estimate} except it uses both $\Dtrain$ and $\mathcal{D}$ to estimate $\pib$. Only $\mathcal{D}$ is used to compute the policy value estimate.
\end{enumerate}
Figure \ref{fig:gridworld-alt} shows that these alternative data sources for estimating $\pib$ decrease accuracy compared to $\RIS$ and $\OIS$.
\textbf{Independent Estimate} has high MSE when the sample size is small but its MSE approaches that of $\OIS$ as the sample size grows.
We understand this result as showing that this baseline cannot correct for sampling error in the off-policy data since the behavior policy estimate is unrelated to the data used in the off-policy evaluation.
\textbf{Extra-data Estimate} initially has high MSE but its MSE decreases faster than that of $\OIS$.
Since this baseline estimates $\pib$ with data that includes $\mathcal{D}$, it can partially correct for sampling error -- though the extra data harms its ability to do so.
Only estimating $\pidata$ with $\mathcal{D}$ and $\mathcal{D}$ alone improves performance over $\OIS$ for all sample sizes.

We also repeat these experiments for the on-policy setting and present results in Figure \ref{fig:gridworld-onpol} and Figure \ref{fig:gridworld-onpol-alt}.
We observe similar trends as in the off-policy experiments suggesting that $\RIS$ can lower variance in Monte Carlo sampling methods even when $\OIS$ weights are otherwise unnecessary.

\begin{figure}[t]
\centering
\subfigure[Gridworld]{\includegraphics[width=0.48\linewidth]{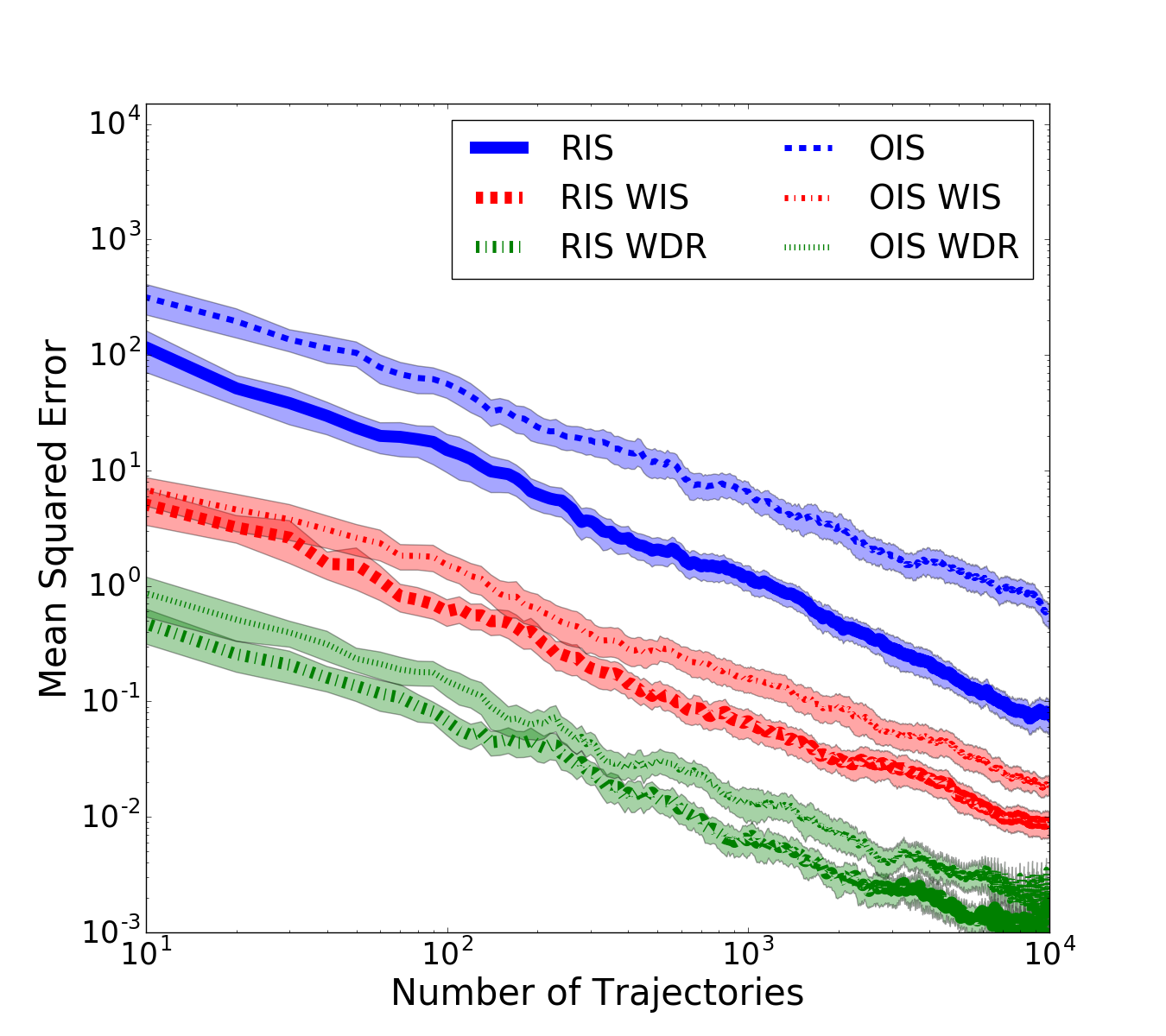}\label{fig:gridworld}}
\subfigure[Gridworld Alt.]{\includegraphics[width=0.48\linewidth]{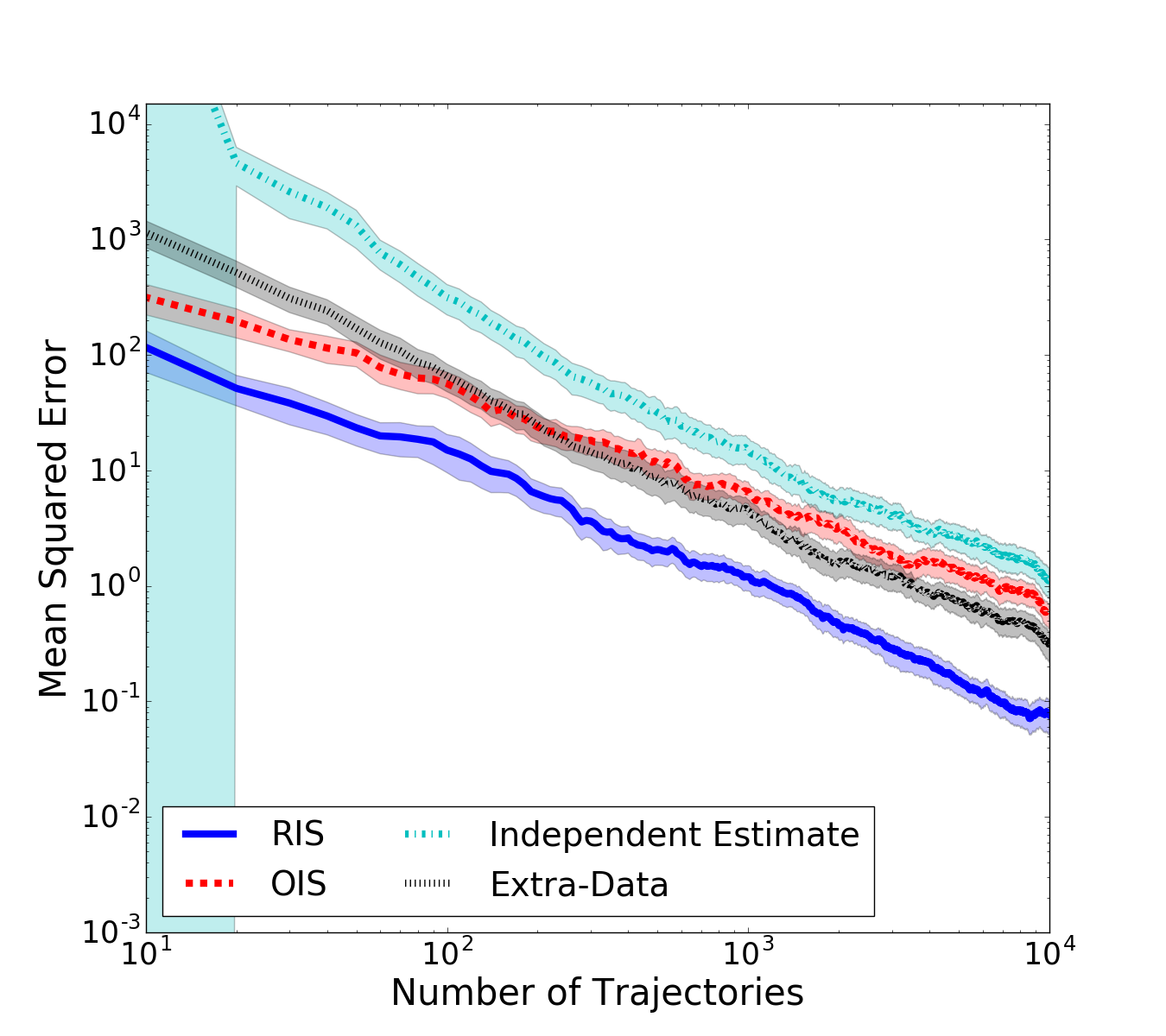}\label{fig:gridworld-alt}}
\subfigure[Gridworld On-Policy]{\includegraphics[width=0.48\linewidth]{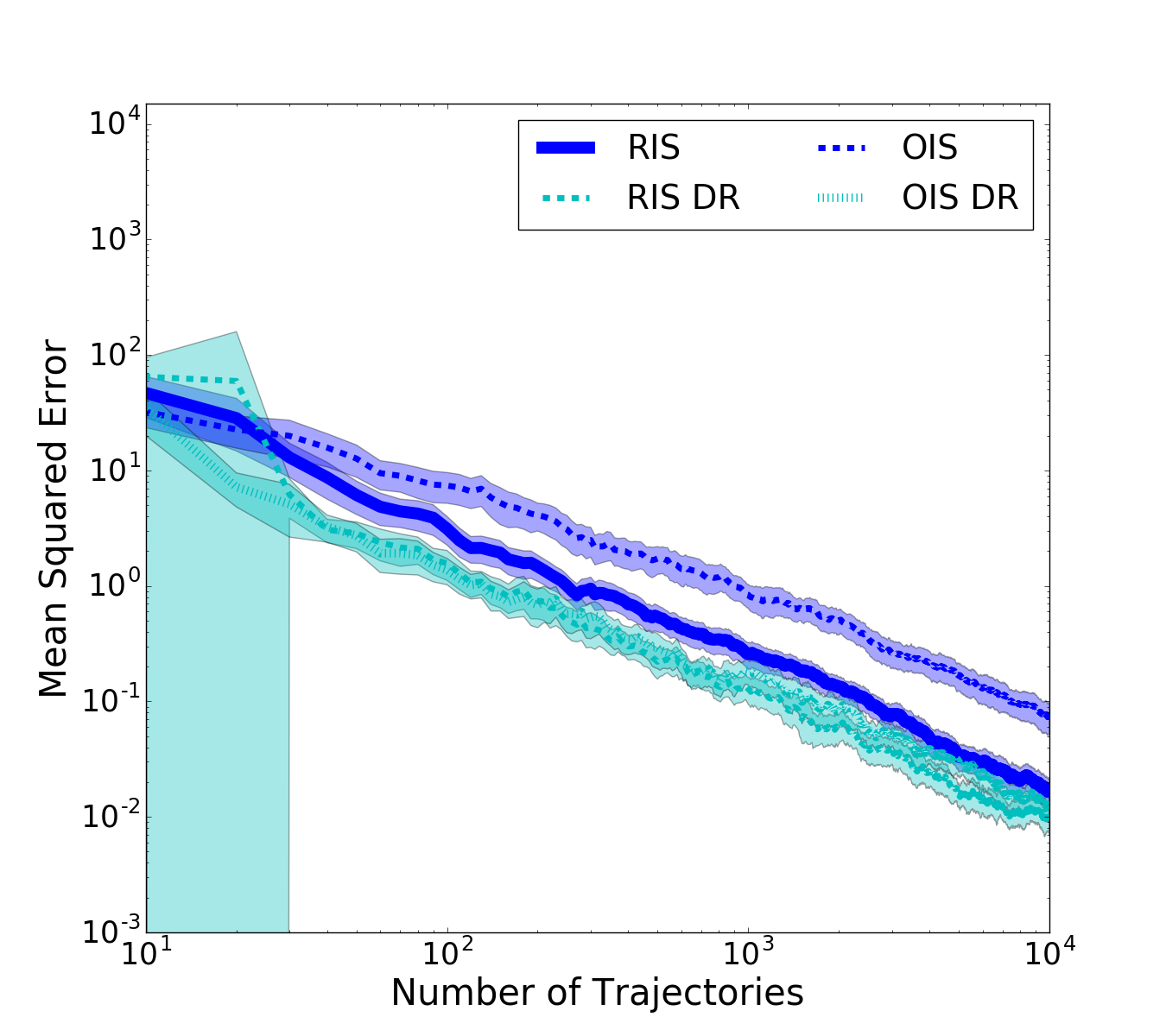}\label{fig:gridworld-onpol}}
\subfigure[Gridworld On-Policy Alt.]{\includegraphics[width=0.48\linewidth]{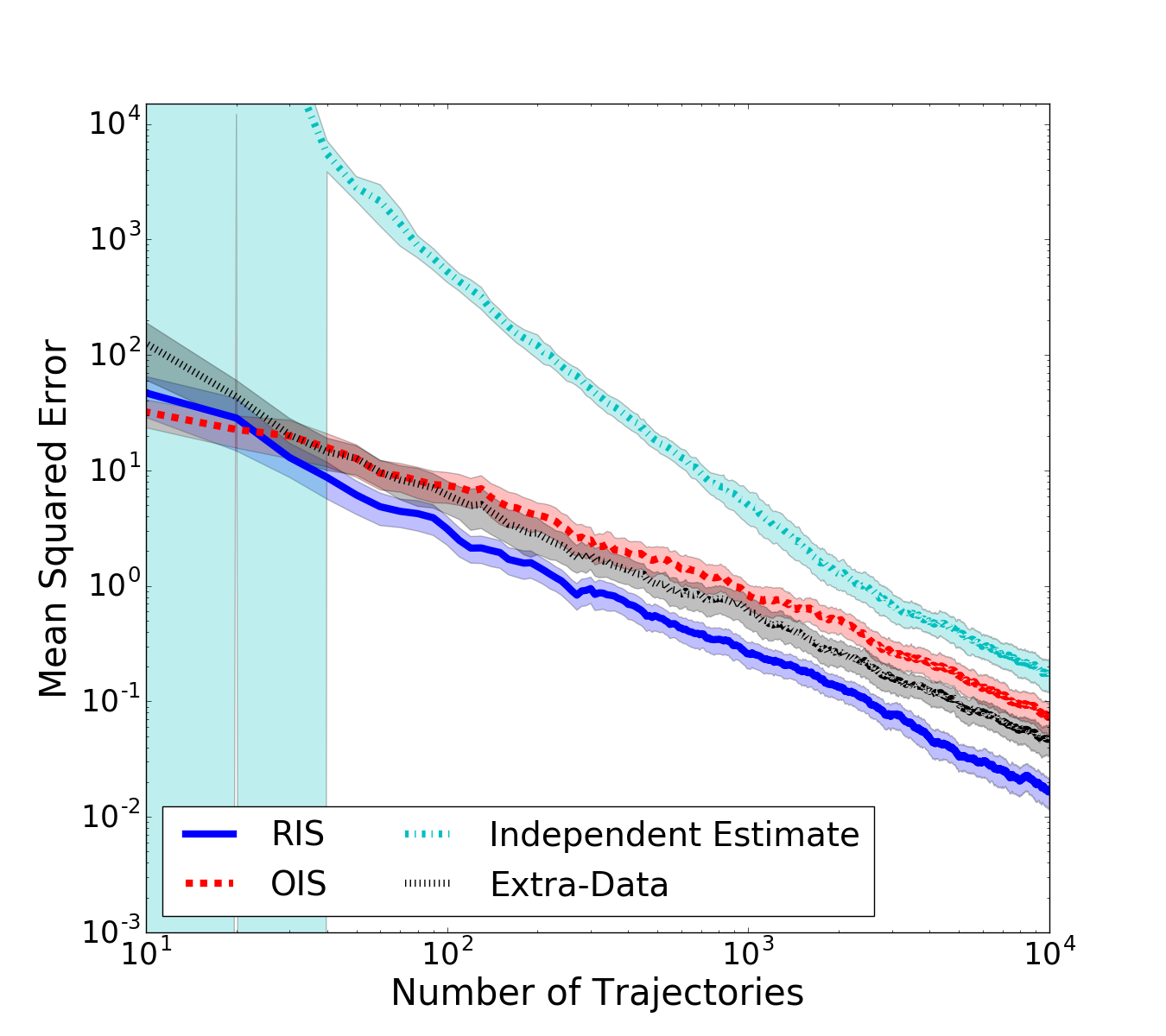}\label{fig:gridworld-onpol-alt}}
\caption{Gridworld policy evaluation results. In all subfigures, the x-axis is the number of trajectories collected and the y-axis is mean squared error. Axes are log-scaled. The shaded region gives a 95\% confidence interval. (a) Gridworld Off-policy Evaluation: The main point of comparison is the $\RIS$ variant of each method to the $\OIS$ variant of each method.
(b) Gridworld $\pidata$ Estimation Alternatives: This plot compares $\RIS$ and $\OIS$ to two methods that replace the true behavior policy with estimates from data sources other than $\mathcal{D}$.
Subfigures (c) and (d) repeat experiments (a) and (b) with the behavior policy from (c) and (d) as the evaluation policy. 
\label{fig:finite}}

\end{figure}

\begin{figure}[t]
\centering
\includegraphics[width=0.7\linewidth]{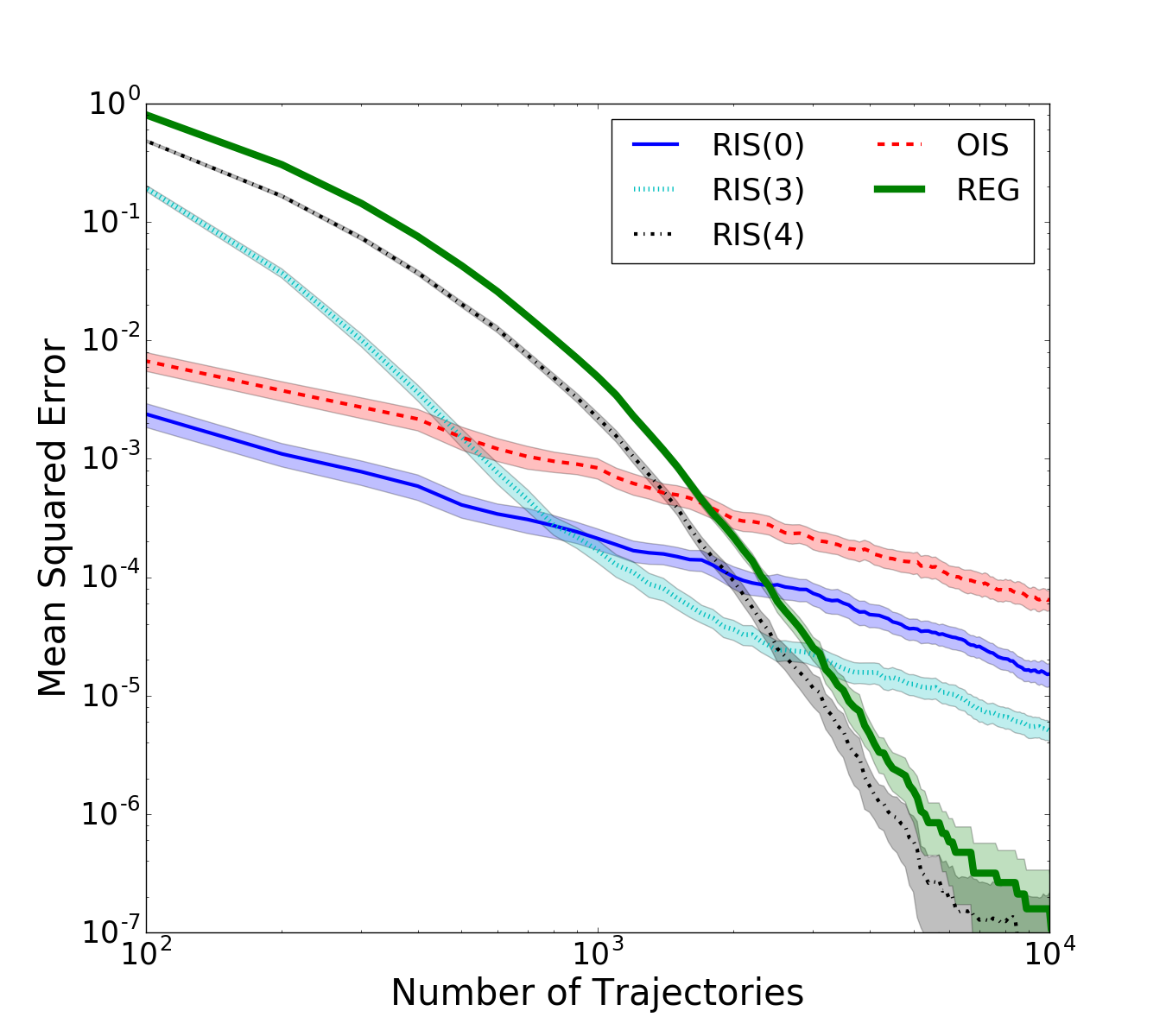}
\caption{Off-policy evaluation in the SinglePath MDP for various $n$. The curves for $\REG$ and $\RIS(4)$ have been cut-off to more clearly show all methods. These methods converge to an MSE value of approximately $1 \times 10^{-31}$.
\label{fig:singlepath-results}}

\end{figure}

\paragraph{RIS(n)}

In the Gridworld domain it is difficult to observe the performance of $\RIS(n)$ for various $n$ because of the long horizon: smaller $n$ perform similarly and larger $n$ scale poorly with $L$.
To see the effects of different $n$ more clearly, we use the SinglePath domain.
Figure \ref{fig:singlepath-results} gives the mean squared error for $\OIS$, $\RIS$, and the $\REG$ estimator of Li et al.\ \yrcite{li2015toward} that has full access to the environment's transition probabilities.
For $\RIS$, we use $n=0, 3, 4$ and each method is ran for $200$ trials.

Figure \ref{fig:singlepath-results} shows that higher values of $n$ and $\REG$ tend to give inaccurate estimates when the sample size is small.
However, as data increases, these methods give increasingly accurate value estimates.
In particular, $\REG$ and $\RIS(4)$ produce estimates with MSE more than 20 orders of magnitude below that of $\RIS(3)$ (Figure \ref{fig:singlepath-results} is cut off at the bottom for clarity of the rest of the results). $\REG$ eventually passes the performance of $\RIS(4)$ since its knowledge of the transition probabilities allows it to eliminate sampling error in both the actions and the environment.
In the low-to-medium data regime, only $\RIS(0)$ outperforms $\OIS$.
However, as data increases, the MSE of all $\RIS$ methods and $\REG$ decreases faster than that of $\OIS$.
The similar performance of $RIS(L-1)$ and $\REG$ supports the connection between these methods that we discuss in Section \ref{sec:RIS:theory}.

\paragraph{RIS with Linear Function Approximation}

Our next set of experiments consider continuous state and action spaces in the Linear Dynamical System domain.
RIS represents $\pidata$ as a Gaussian policy with mean given as a linear function of the state features.
Similar to in Gridworld, we compare three variants of IS, each implemented with $\RIS$ and $\OIS$ weights: the ordinary IS estimator, weighted IS (WIS), and per-decison IS (PDIS).
Each method is averaged over $200$ trials and results are shown in Figure \ref{fig:lds-baselines}.

\begin{figure}
\centering
\subfigure[LDS]{\includegraphics[width=0.48\linewidth]{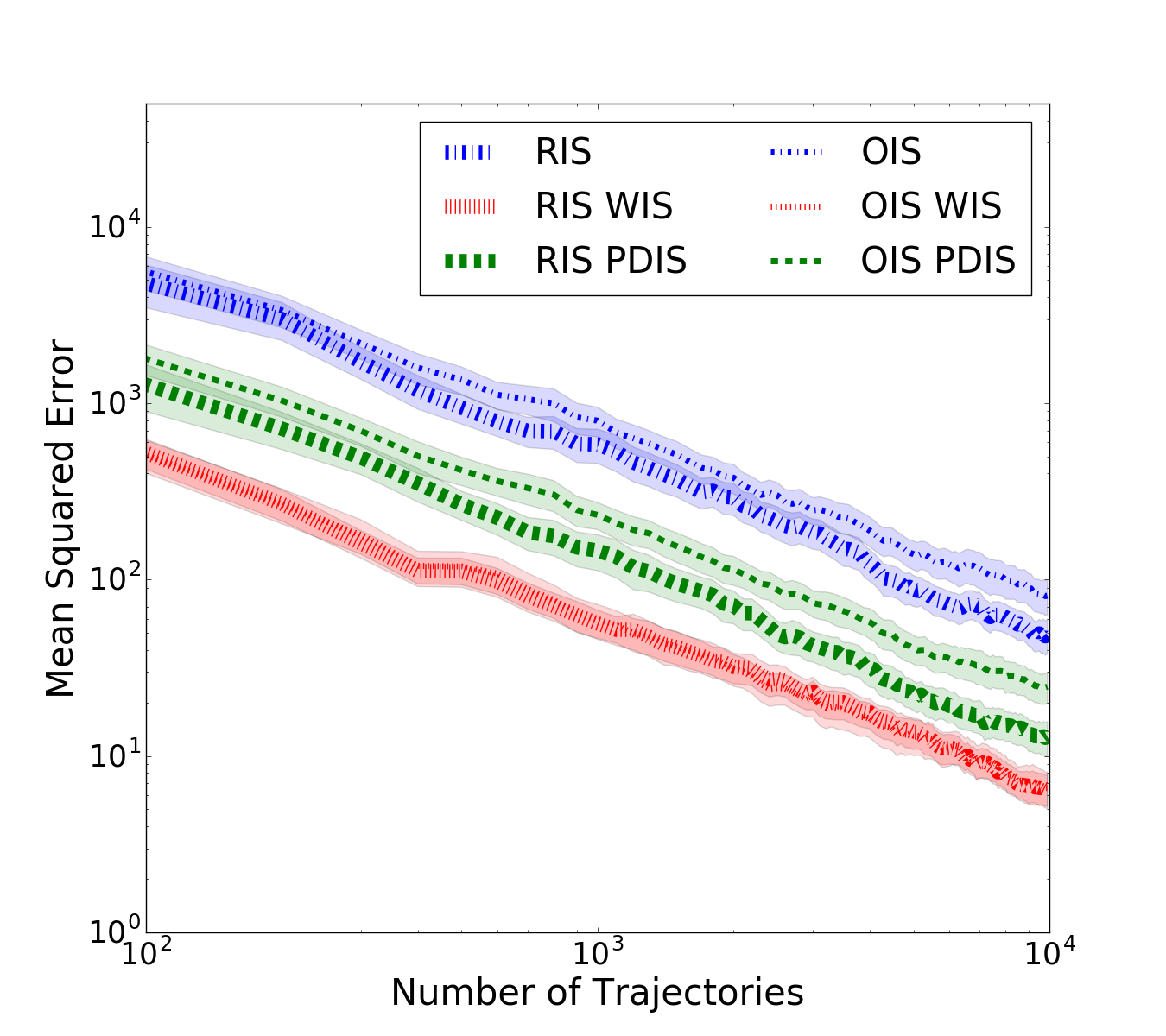}\label{fig:lds-baselines}}
\subfigure[LDS Alt. Weights]{\includegraphics[width=0.48\linewidth]{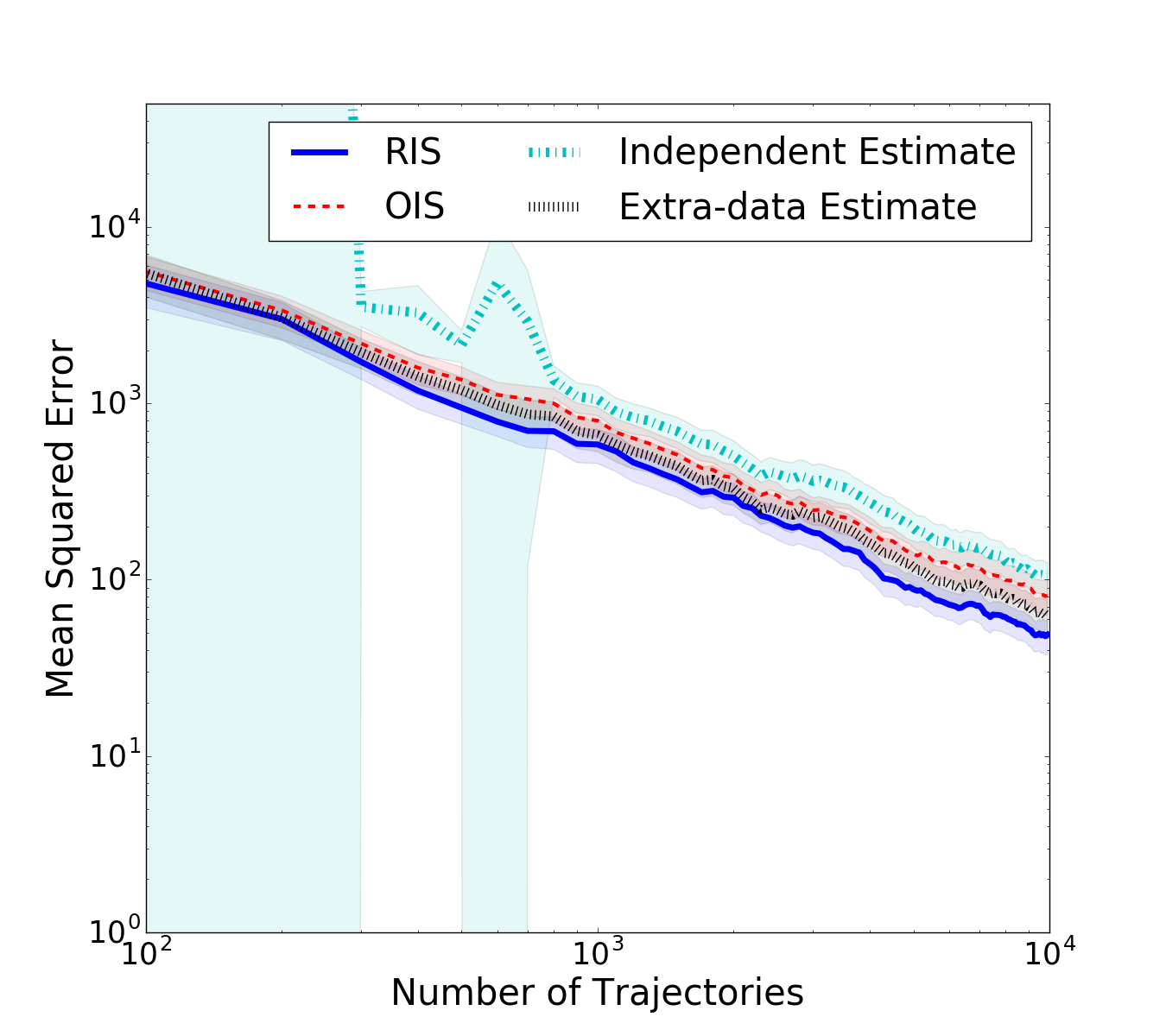}\label{fig:lds-alternatives}}
\caption{Linear dynamical system results. Figure \ref{fig:lds-baselines} shows the mean squared error (MSE) for three $\IS$ variants with and without RIS weights. Figure \ref{fig:lds-alternatives} shows the MSE for different methods of estimating the behavior policy compared to $\RIS$ and $\OIS$. Axes and scaling are the same as in Figure \ref{fig:gridworld}.}
\end{figure}

We see that $\RIS$ weights improve both IS and PDIS, while both WIS variants have similar MSE.
This result suggests that the MSE improvement from using $\RIS$ weights depends, at least partially, on the variant of $\IS$ being used.

Similar to Gridworld, we also consider estimating $\pidata$ with either an independent data-set or with extra data and see a similar ordering of methods.
%
%
\textbf{Independent Estimate} gives high variance estimates for small sample sizes but then approaches $\OIS$ as the sample size grows.
\textbf{Extra-Data Estimate} corrects for some sampling error and has lower MSE than $\OIS$.
$\RIS$ lowers MSE compared to all baselines.


\paragraph{RIS with Neural Networks}

Our remaining experiments use the Hopper and HalfCheetah domains.
RIS represents $\pidata$ as a neural network that maps the state to the mean of a Gaussian distribution over actions.
The standard deviation of the Gaussian is given by state-independent parameters.
In these experiments, we sample a single batch of $400$ trajectories and compare the MSE of $\RIS$ and $\IS$ on this batch.
We repeat this experiment $200$ times for each method.

Figure \ref{fig:bar} compares the MSE of $\RIS$ for different neural network architectures.
Our main point of comparison is $\RIS$ using the architecture that achieves the lowest validation error during training (the darker bars in Figure \ref{fig:bar}).
Under this comparison, the MSE of $\RIS$ with a two hidden layer network is lower than that of $\OIS$ in both Hopper and HalfCheetah though, in HalfCheetah, the difference is statistically insignificant.
We also observe that the policy class with the best validation error does \textit{not} always give the lowest MSE (e.g., in Hopper, the two hidden layer network gives the lowest validation loss but the network with a single layer of hidden units has $\approx25$\% less MSE than the two hidden layer network).
This last observation motivates our final experiment.

\paragraph{RIS Model Selection}

\begin{figure}
\centering
\subfigure[Hopper]{\includegraphics[width=0.45\linewidth]{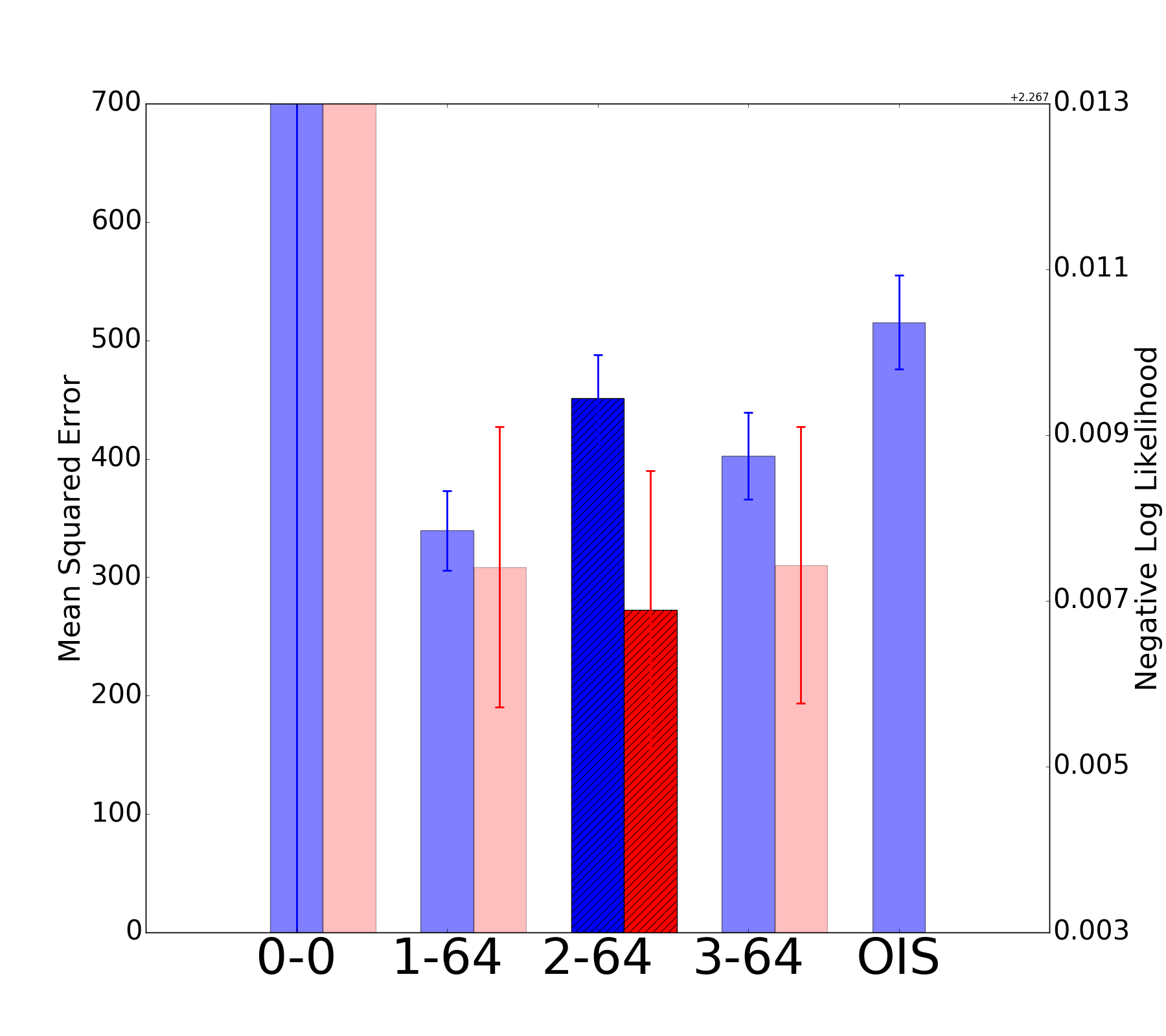}\label{fig:hopper}}
\subfigure[HalfCheetah]{\includegraphics[width=0.45\linewidth]{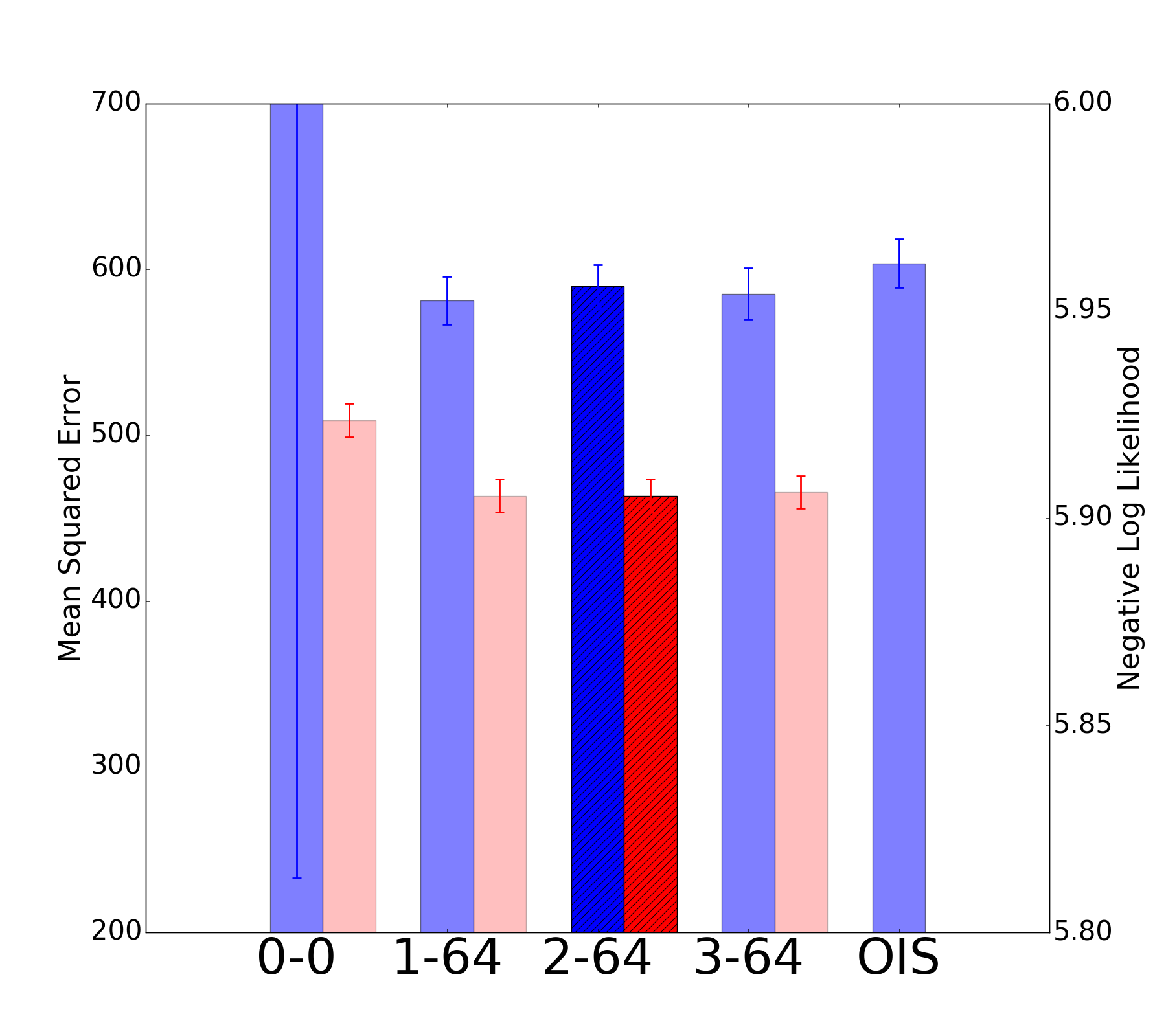}\label{fig:cheetah}}
\caption{Figures \ref{fig:hopper} and \ref{fig:cheetah} compare different neural network architectures (specified as \#-layers-\#-units) for regression importance sampling on the Hopper and HalfCheetah domain. The darker, blue bars give the MSE for each architecture and $\OIS$. Lighter, red bars give the negative log likelihood of a hold-out data set. Our main point of comparison is the MSE of the architecture with the lowest hold-out negative log likelihood (given by the darker pair of bars) compared to the MSE of IS.
\label{fig:bar}}
\end{figure}

Our final experiment aims to better understand how hold-out validation error relates to the MSE of the RIS estimator when using gradient descent to estimate neural network approximations of $\pidata$.
This experiment duplicates our previous experiment, except every 25 steps of gradient descent we stop optimizing $\pidata$ and compute the $\RIS$ estimate with the current $\pidata$ and its $\mse$.
We also compute the training and hold-out validation negative log-likelihood.
Plotting these values gives a picture of how the $\mse$ of $\RIS$ changes as our estimate of $\pidata$ changes.
Figure \ref{fig:fit_exp} shows this plot for the Hopper domain.

We see that the policy with minimal $\mse$ and the policy that minimizes validation loss are misaligned.
If training is stopped when the validation loss is minimized, the $\mse$ of $\RIS$ is lower than that of $\OIS$ (the intersection of the RIS curve and the vertical dashed line in Figure \ref{fig:fit_exp}.
However, the $\pidata$ that minimizes the validation loss curve is \textit{not} identical to the $\pidata$ that minimizes $\mse$.

To understand this result, we also plot the average $\RIS$ estimate throughout behavior policy learning (bottom of Figure \ref{fig:fit_exp}).
We can see that at the beginning of training, $\RIS$ tends to \textit{over-estimate} $v(\pieval)$ because the probabilities given by $\pidata$ to the observed data will be small (and thus the $\RIS$ weights are large).
As the likelihood of $\mathcal{D}$ under $\pidata$ increases (negative log likelihood decreases), the $\RIS$ weights become smaller and the estimates tend to \textit{under-estimate} $v(\pieval)$.
%
%
%
%
%
The implication of these observations, for $\RIS$, is that during behavior policy estimation the $\RIS$ estimate will likely have zero $\mse$ at some point.
Thus, there may be an early stopping criterion -- besides minimal validation loss -- that would lead to lower $\mse$ with $\RIS$, however, to date we have not found one.
Note that $\OIS$ also tends to under-estimate policy value in MDPs as has been previously analyzed by Doroudi et al. \yrcite{doroudi2017importance}.
Appendix F shows the same observations in the HalfCheetah domain.

\begin{figure}[t!]
\centering
\includegraphics[width=0.85\linewidth]{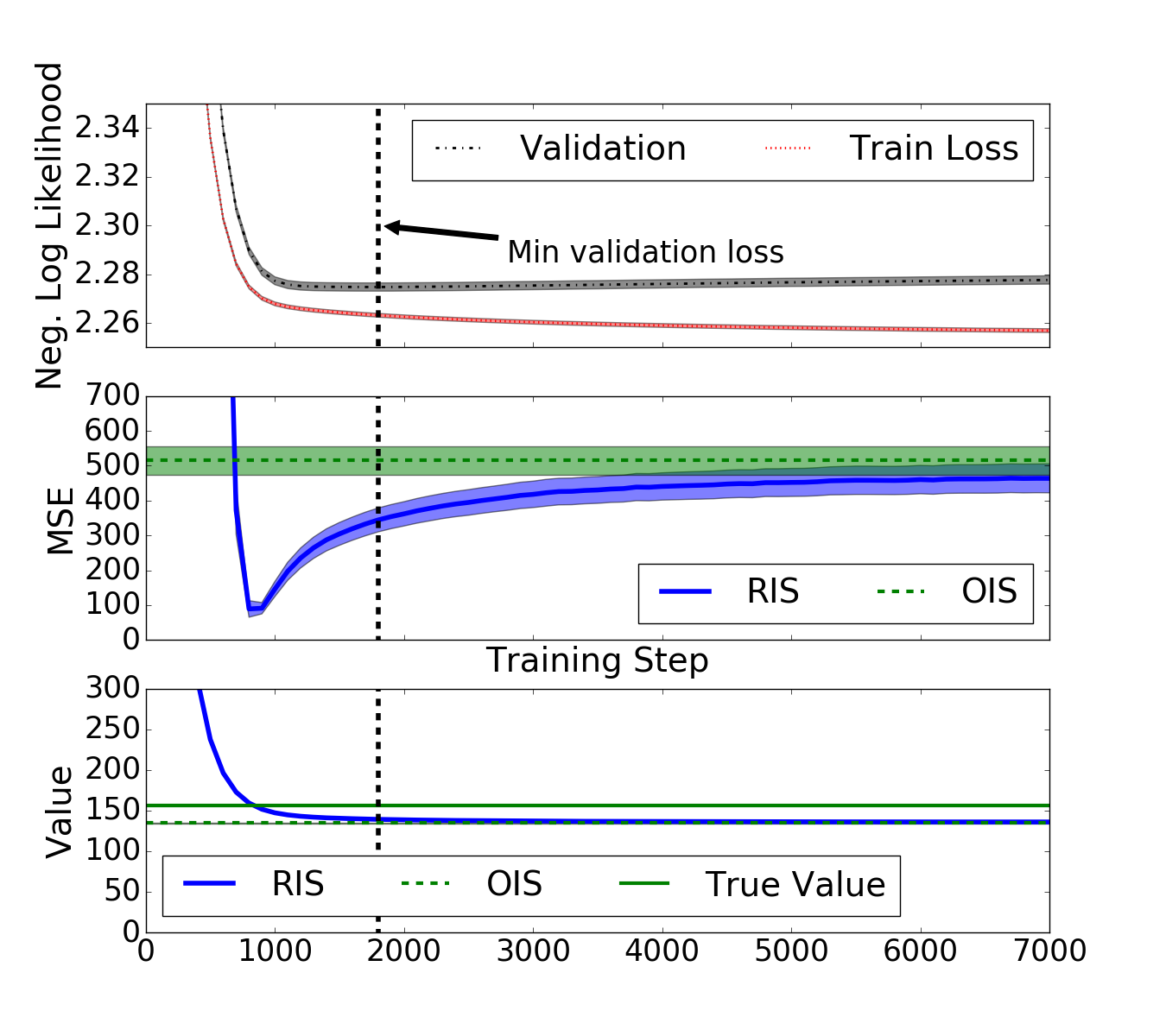}
\caption{Mean squared error and estimate of the importance sampling estimator during training of $\pidata$. The x-axis is the number of gradient descent steps. The top plot shows the training and validation loss curves. The y-axis of the top plot is the average negative log-likelihood. 
The y-axis of the middle plot is mean squared error (MSE).
The y-axis of the bottom plot is the value of the estimate.
MSE is minimized close to, but slightly before, the point where the validation and training loss curves indicate that overfitting is beginning. 
This point corresponds to where the $\RIS$ estimate transitions from over-estimating to under-estimating.\label{fig:fit_exp}}
\end{figure}

\section{Related Work}

In this section we survey work related to behavior policy estimation for importance sampling.
Methods related to $\RIS$ have been studied for Monte Carlo integration \cite{henmi2007importance,delyon2016integral} and causal inference \cite{hirano2003efficient,rosenbaum1987model}.
The $\REG$ method (discussed below) can be seen as the direct extension of these methods to MDPs.
In contrast to these works, we study policy evaluation in Markov decision processes which introduces sequential structure into the samples and unknown stochasticity in the state transitions.
%
%
%
These methods have also, to the best of our knowledge, \textit{not} been studied in Markov decision processes or for sequential data.

Li et al.\ \yrcite{li2015toward} study the \emph{regression} (REG) estimator for off-policy evaluation and show that its minimax MSE is asymptotically optimal though it might perform poorly for small sample sizes.
Though $\REG$ and $\RIS$ are equivalent for multi-armed bandit problems, for MDPs, the definition of $\REG$ and any $\RIS$ method diverge. 
%
Figure \ref{fig:singlepath-results} shows that all tested $\RIS$ methods improve over $\REG$ for small sample sizes though $\REG$ has lower asymptotic $\mse$.
Intuitively, $\REG$ corrects for sampling error in both the action selection and state transitions through knowledge of the true state-transition function.
However, such knowledge is usually unavailable and, in these cases, $\REG$ is inapplicable. 

Narita et al.\ \yrcite{narita2019efficient} study behavior policy estimation for policy evaluation and improvement in multi-armed bandit problems.
They also show lower asymptotic variance (as we do), however, their results are only for the bandit setting.

In the contextual bandit literature, Dudik et al.\ \yrcite{dudik2011doubly} present finite sample bias and variance results for importance sampling that is applicable when the behavior policy probabilities are different than the true behavior policy.
Farajtabar et al.\ \yrcite{farajtabar2018more} extended these results to full MDPs.
These works make the assumption that $\pidata$ is estimated independently from the data used in the final $\IS$ evaluation.
In contrast, $\RIS$ uses the same set of data to both estimate $\pib$ and compute the $\IS$ evaluation.
This choice allows $\RIS$ to correct for sampling error and improve upon the $\OIS$ estimate (as shown in Figure \ref{fig:gridworld-alt}, \ref{fig:gridworld-onpol-alt}, and \ref{fig:lds-alternatives}).

A large body of work exists on lowering the variance of importance sampling for off-policy evaluation.
Such approaches include control variates \cite{jiang2016doubly,thomas2016data-efficient}, normalized importance weights \cite{precup2000eligibility,swaminathan2015self}, and importance ratio clipping \cite{bottou2013counterfactual}.
These variance reduction strategies are complementary to regression importance sampling; any of these methods can be combined with $\RIS$ for further variance reduction.

\section{Discussion and Future Work}

Our experiments demonstrate that regression importance sampling can obtain lower mean squared error than ordinary importance sampling for off-policy evaluation in Markov decision process environments.
The main practical conclusion of our paper is the importance of estimating $\pidata$ with the same data used to compute the importance sampling estimate.
We also demonstrate that estimating a behavior policy that conditions on trajectory segments -- instead of only the preceding state -- improves performance in the large sample setting.

%
For all $n$, $\RIS(n)$ is consistent and has lower asymptotic variance than $\OIS$.
There remain theoretical questions concerning the finite-sample setting and relaxing the assumption that we estimate $\pidata$ from a policy class that includes the true behavior policy.
The connection to the $\REG$ estimator and our empirical results suggest that $\RIS$ with $n$ close to $L$ may suffer from high bias.
Future work that quantifies or bounds this bias will give us a better understanding of RIS methods.
Relaxing the assumption that $\pib \in \Pi$ or analyzing the case when $\pib \not \in \Pi$ is also an important next step for bridging the gap between our presented theory and the use of $\RIS$ in settings where the policy class of $\pib$ is unknown.

In this paper we focused on \textit{batch} policy evaluation where $\mathcal{D}$ is given and fixed.
Studying $\RIS$ for \textit{online} policy evaluation setting is an interesting direction for future work.
Finally, incorporating $\RIS$ into policy improvement methods is an interesting direction for future work.
In work parallel to our own, two of the authors \cite{hanna2019reducing} explored using an estimated behavior policy to lower sampling error in on-policy policy gradient learning.
However, our approach in that paper only focuses on reducing variance in the one-step action selection while $\RIS$ could lower variance in the full return estimation.
%

\section{Conclusion}

We have studied a class of off-policy evaluation importance sampling methods, called regression importance sampling methods, that apply importance sampling after first estimating the behavior policy that generated the data.
Notably, RIS estimates the behavior policy from the same set of data that is also used for the IS estimate.
Computing the behavior policy estimate and IS estimate from the same set of data allows RIS to correct for the sampling error inherent to importance sampling with the true behavior policy.
We evaluated RIS across several policy evaluation tasks and show that it improves over ordinary importance sampling -- that uses the true behavior policy -- in several off-policy policy evaluation tasks.
Finally, we showed that, as the sample size grows, it can be beneficial to ignore knowledge that the true behavior policy is Markovian.

{\small
\section*{Acknowledgments}
We would like to thank Garrett Warnell, Ishan Durugkar, Philip Thomas, Qiang Liu, Faraz Torabi, Leno Felipe da Silva, Marc Bellemare, Finale Doshi-Velez and the anonymous reviewers for insightful comments that suggested new directions to study and improved the final presentation of the work. This work has taken place in the Learning Agents Research
Group (LARG) and the Personal Autonomous Robots Lab (PEARL) at the Artificial Intelligence Laboratory, The University
of Texas at Austin.  LARG research is supported in part by NSF
(IIS-1637736, IIS-1651089, IIS-1724157), ONR (N00014-18-2243), FLI
(RFP2-000), ARL, DARPA, Intel, Raytheon, and Lockheed Martin. PeARL research is supported in part by the NSF (IIS-1724157, IIS-1638107, IIS-1617639, IIS-1749204) and ONR(N00014-18-2243).
Josiah Hanna is supported by an IBM PhD Fellowship.
Peter Stone serves on the Board of Directors of Cogitai, Inc.  The
terms of this arrangement have been reviewed and approved by the
University of Texas at Austin in accordance with its policy on
objectivity in research.
}

\bibliographystyle{icml2019}

\onecolumn

\begin{appendix}

%
%
%

\section{Regression Importance Sampling is Consistent}

In this appendix we show that the regression importance sampling (RIS) estimator is a consistent estimator of $v(\pieval)$ under two assumptions.
The main intuition for this proof is that $\RIS$ is performing policy search on an estimate of the log-likelihood, $\widehat{\loss}(\pi | \data)$, as a surrogate objective for the true log-likelihood, $\loss(\pi)$.
Since $\pib$ has generated our data, $\pib$ is the optimal solution to this policy search.
As long as, for all $\pi$, $\widehat{\loss}(\pi | \data)$ is a consistent estimator of $\loss(\pi)$ then selecting $\pidata = \displaystyle \argmax_\pi \widehat{\loss}(\pi | \data)$ will converge probabilistically to $\pib$ and the $\RIS$ estimator will be the same as the $\OIS$ estimator which is a consistent estimator of $v(\pieval)$.
If the set of policies we search over, $\Pi$, is countable then this argument is almost enough to show $\RIS$ to be consistent.
The difficulty (as we explain below) arises when $\Pi$ is \textit{not} countable.

Our proof takes inspiration from Thomas and Brunskill who show that their Magical Policy Search algorithm converges to the optimal policy by maximizing a surrogate estimate of policy value \yrcite{thomas2016magical}.
They show that performing policy search on a policy value estimate, $\hat{v}(\pi)$, will almost surely return the policy that maximizes $v(\pi)$ if $\hat{v}(\pi)$ is a consistent estimator of $v(\pi)$.
The proof is almost identical; the notable difference is substituting the log-likelihood, $\loss(\pi)$, and a consistent estimator of the log-likelihood, $\widehat{\loss}(\pi | \data)$, in place of $v(\pi)$ and $\hat{v}(\pi)$.

\subsection{Definitions and Assumptions}

Let $(\Omega, \mathcal{F}, \mu)$ be a probability space and $D_m: \Omega \rightarrow \mathcal{D}$ be a random variable.
$D_m(\omega)$ is a sample of $m$ trajectories with $\omega \in \Omega$.
Let $d_\pib$ be the distribution of states under $\pib$.
Define the expected log-likelihood:
 \[\loss(\pi) = \mathbf{E}\left[\log \pi(A | S) \middle| S \sim d_\pib, A \sim \pib\right]\] 
 and its sample estimate from samples in $D_m(\omega)$: 
\[ \widehat{\loss}(\pi | D_m(\omega)) = \frac{1}{mL} \sum_{H \in D_m(\omega)} \sum_{t=0}^{L-1} \log \pi(A_t^{H} | S_t^{H}). \]
where $S_t^H$ and $A_t^H$ are the random variables representing the state and action that occur at time-step $t$ of trajectory $H$.

Assuming for all $s,a$ the variance of $\log \pi(a | s)$ is bounded, $\widehat{\loss}(\pi | D_m(\omega))$ is a consistent estimator of $\loss(\pi)$.
We make this assumption explicit:
\begin{ass}
(Consistent Estimation of Log likelihood). For all $\pi \in \Pi$, $\widehat{\loss}(\pi | D_m(\omega)) \xrightarrow{a.s.} \loss(\pi)$. 
\end{ass}
This assumption will hold when the support of $\pib$ is a subset of the support of $\pi$ for all $\pi \in \Pi$, i.e., no $\pi \in \Pi$ places zero probability measure on an action that $\pib$ might take.
%
%
We can ensure this assumption is satisfied by only considering $\pi \in \Pi$ that place non-zero probability on any action that $\pib$ has taken.

We also make an additional assumption about the piece-wise continuity of the log-likelihood, $\loss$, and the estimate of the log-likelihood, $\widehat{\loss}$.
First we present two necessary definitions as given by Thomas and Brunskill \yrcite{thomas2016magical}:

\begin{defn}
(Piecewise Lipschitz continuity). We say that a function $f: M \rightarrow \mathbb{R}$ on a metric space $(M, d)$ is piecewise Lipschitz continuous with respect to Lipschitz constant $K$ and with respect to a countable partition, $\{M_1,M_2,...\}$ if $f$ is Lipschitz continuous with Lipschitz constant $K$ on all metric spaces in $\{(M_i, d_i)\}_{i=1}^\infty$.
\end{defn}

\begin{defn}
($\delta$-covering). If $(M,d)$ is a metric space, a set $X \subset M$ is a $\delta$-covering of $(M,d)$ if and only if $\max_{y \in M} \min_{x \in X} d(x,y) \leq \delta$.
\end{defn}

We now present our final assumption:

\begin{ass}(Piecewise Lipschitz objectives). Our policy class, $\Pi$, is equipped with a metric, $d_\Pi$, such that for all $D_m(\omega)$ there exist countable partition of $\Pi$, $\Pi^\loss \coloneqq \{\Pi^\loss_1, \Pi^\loss_2, ...\}$ and $\Pi^{\widehat{\loss}} \coloneqq \{\Pi^{\widehat{\loss}}_1, \Pi^{\widehat{\loss}}_2, ...\}$, where $\loss$ and $\widehat{\loss}(\cdot | D_m(\omega))$ are piecewise Lipschitz continuous with respect to $\Pi^\loss$ and $\Pi^{\widehat{\loss}}$ with Lipschitz constants $K$ and $\widehat{K}$ respectively. 
Furthermore, for all $i \in \mathbb{N}_{>0}$ and all $\delta > 0$ there exist countable $\delta$-covers of $\Pi^\loss_i$ and $\Pi^{\widehat{\loss}}_i$.
\end{ass}

As pointed out by Thomas and Brunskill, this assumption holds for the most commonly considered policy classes but is also general enough to hold for other settings (see Thomas and Brunskill \yrcite{thomas2016magical} for further discussion of Assumptions 1 and 2 and the related definitions).

\subsection{Consistency Proof}

Note that: 
\[ \pib = \argmax_{\pi \in \Pi} \loss(\pi)
\]
\[\pidata = \argmax_{\pi \in \Pi} \widehat{\loss}(\pi | D_m(\omega)).\]
Define the KL-divergence ($D_\mathtt{KL})$) between $\pib$ and $\pidata$ in state $s$ as: $\delta_\mathtt{KL}(s) = D_\mathtt{KL}(\pib(\cdot|s), \pidata(\cdot | s))$.

\begin{lemma}\label{lemma:consistent}
If Assumptions 1 and 2 hold then $\mathbf{E}_{d_\pib}[\delta_\mathtt{KL}(s)] \xrightarrow{a.s.} 0$.
\end{lemma}

\begin{proof}
Define $\Delta(\pi, \omega) = |\widehat{\loss}(\pi | D_m(\omega)) - \loss(\pi)|$.
From Assumption 1 and one definition of almost sure convergence, for all $\pi \in \Pi$ and for all $\epsilon > 0$:
\begin{equation}\label{eq:consistency:1}
\Pr\left( \liminf_{m\rightarrow\infty} \{ \omega \in \Omega: \Delta(\pi, \omega) < \epsilon \}\right) = 1.
\end{equation}

Thomas and Brunskill point out that because $\Pi$ may not be countable, (\ref{eq:consistency:1}) may not hold at the same time for all $\pi \in \Pi$. 
More precisely, it does \textit{not} immediately follow that for all $\epsilon >0$:
{
\begin{equation}\label{eq:consistency:2}
\Pr\left( \liminf_{m\rightarrow\infty} \{ \omega \in \Omega: \forall \pi \in \Pi, \Delta(\pi, \omega) < \epsilon \}\right) = 1.
\end{equation}
}

Let $C(\delta)$ denote the union of all of the policies in the $\delta$-covers of the countable partitions of $\Pi$ assumed to exist by Assumption 2. 
Since the partitions are countable and the $\delta$-covers for each region are assumed to be countable, we have that $C(\delta)$ is countable for all $\delta$. 
Thus, for all $\pi \in C(\delta)$, (\ref{eq:consistency:1}) holds simulatenously.
More precisely, for all $\delta > 0$ and for all $\epsilon > 0$:

{\small
\begin{equation} \label{eq:consistency:3}
\Pr\left( \liminf_{m\rightarrow\infty} \{ \omega \in \Omega: \forall \pi \in C(\delta), \Delta(\pi, \omega) < \epsilon \}\right) = 1.
\end{equation}
}

Consider a $\pi \not\in C(\delta)$. By the definition of a $\delta$-cover and Assumption 2, we have that $\exists\pi^\prime \in \Pi_i^\loss, d(\pi, \pi^\prime) \leq \delta$.
Since Assumption 2 requires $\loss$ to be Lipschitz continuous on $\Pi^\loss_i$, we have that $|\loss(\pi) - \loss(\pi^\prime)| \leq K\delta$. 
Similarly $|\widehat{\loss}(\pi | D_m(\omega)) - \widehat{\loss}(\pi^\prime | D_m(\omega))| \leq \widehat{K}\delta$.
So, $|\widehat{\loss}(\pi | D_m(\omega)) - \loss(\pi)| \leq |\widehat{\loss}(\pi | D_m(\omega)) - \loss(\pi^\prime)| + K\delta \leq |\widehat{\loss}(\pi^\prime | D_m(\omega)) - \loss(\pi^\prime)| + (\widehat{K} + K)\delta$.
Then it follows that for all $\delta > 0$:
\begin{align*}
\left( \forall \pi \in C(\delta), \Delta(\pi, \omega) \leq \epsilon \right) \rightarrow \\
 \left( \forall \pi \in \Pi, \Delta(\pi, \omega) < \epsilon + (K + \widehat{K})\delta\right).
\end{align*}
Substituting this into (\ref{eq:consistency:3}) we have that for all $\delta > 0$ and for all $\epsilon > 0$:

{\small
\[ \Pr\left( \liminf_{m\rightarrow\infty} \{\omega \in \Omega: \forall \pi \in \Pi, \Delta(\pi, \omega) < \epsilon + (K + \widehat{K}) \delta\} \right) = 1\]
}

The next part of the proof massages (\ref{eq:consistency:3}) into a statement of the same form as (\ref{eq:consistency:2}).
Consider the choice of $\delta \coloneqq \epsilon / (K + \widehat{K})$.
Define $\epsilon^\prime = 2 \epsilon$.
Then for all $\epsilon^\prime > 0$:

{
\begin{equation}\label{eq:3}
\Pr\left( \liminf_{m\rightarrow\infty} \{\omega \in \Omega: \forall \pi \in \Pi, \Delta(\pi, \omega) < \epsilon^\prime\} \right) = 1
\end{equation}
}

%
Since $\forall \pi \in \Pi, \Delta(\pi, \omega) < \epsilon^\prime$, we obtain:
\begin{equation}\label{eq:4}
\Delta(\pib, \omega) < \epsilon^\prime
\end{equation}
\begin{equation}\label{eq:5}
\Delta(\pidata, \omega) < \epsilon^\prime
\end{equation}
and then applying the definition of $\Delta$
\begin{align}
\loss(\pidata) \stackrel{(a)}{\leq}& \loss(\pib) \label{eq6} \\ 
\stackrel{(b)}{<}& \widehat{\loss}(\pib | D_m(\omega)) + \epsilon^\prime \\
\stackrel{(c)}{\leq}& \widehat{\loss}(\pidata|D_m(\omega)) + \epsilon^\prime \\
\stackrel{(d)}{\leq}& \loss(\pidata) + 2 \epsilon^\prime \label{eq7}
\end{align}

where (a) comes from the fact that $\pib$ maximizes $\loss$, (b) comes from (\ref{eq:4}), (c) comes from the fact that $\pidata$ maximizes $\widehat{\loss}(\cdot | D_m(\omega))$, and (d) comes from (\ref{eq:5}). Considering (\ref{eq6}) and (\ref{eq7}), it follows that $| \loss(\pidata) - \loss(\pib)| < 2\epsilon^\prime$. Thus, (\ref{eq:3}) implies that:

\[\forall \epsilon^\prime > 0, \Pr\left( \liminf_{m\rightarrow\infty} \{\omega \in \Omega: | \loss(\pidata) - \loss(\pib)| < 2\epsilon^\prime \} \right) = 1\]

Using $\epsilon'' \coloneqq 2\epsilon^\prime$ we obtain:
\[\forall \epsilon'' > 0, \Pr\left( \liminf_{m\rightarrow\infty} \{\omega \in \Omega: | \loss(\pidata) - \loss(\pib)| < \epsilon'' \} \right) = 1\]
From the definition of the KL-Divergence, 
\[
\loss(\pidata) - \loss(\pib) = \mathbf{E}_{d_\pib}[\delta_\mathtt{KL}(s)] \]
and we obtain that:
\[\forall \epsilon > 0, \Pr \left(  \liminf_{n \rightarrow \infty} \{ \omega \in \Omega: | - \mathbf{E}_{d_\pib}[\delta_\mathtt{KL}(s)] | < \epsilon \} \right) = 1\]
And finally, since the KL-Divergence is non-negative:
\[\forall \epsilon > 0, \Pr \left(  \liminf_{m\rightarrow\infty} \{\omega \in \Omega: \mathbf{E}_{d_\pib}[\delta_\mathtt{KL}(s)]| < \epsilon \} \right) = 1,\]

which, by the definition of almost sure convergence, means that $\mathbf{E}_{d_\pib}[\delta_\mathtt{KL}(s)] \xrightarrow{a.s.} 0$.
\end{proof}

\begin{proposition}
If Assumptions 1 and 2 hold, then $\RIS(n)$ is a consistent estimator of $v(\pieval)$: $\RIS(n)(\pieval, \mathcal{D}) \xrightarrow{a.s.} v(\pieval)$.
\end{proposition}

\begin{proof}
Lemma \ref{lemma:consistent} shows that as the amount of data increases, the behavior policy estimated by $\RIS$ will almost surely converge to the true behavior policy.
Almost sure convergence to the true behavior policy means that $\RIS$ almost surely converges to the ordinary $\OIS$ estimate.
Since $\OIS$ is a consistent estimator of $v(\pieval)$, $\RIS$ is also a consistent estimator of $v(\pieval)$.
\end{proof}

\section{Asymptotic Variance Proof}

In this appendix we prove that, $\forall n$, $\RIS(n)$ has asymptotic variance at most that of OIS. 
We give this result as a corollary to Theorem 1 of Henmi et al.\ \yrcite{henmi2007importance} that holds for general Monte Carlo integration. 
Note that while we define distributions as probability mass functions, this result can be applied to continuous-valued state and action spaces by replacing probability mass functions with density functions.
\begin{corollary}
Let $\Pi_\btheta^n$ be a class of twice differentiable policies, $\pitheta(\cdot | s_{t-n}, a_{t-n},\dots, s_t)$. If $\exists \tilde{\btheta}$ such that $\pi_{\tilde{\btheta}} \in \Pi_\btheta^n$ and $\pi_{\tilde{\btheta}} = \pib$ then \[\var_A({\RIS(n)(\pieval, \mathcal{D})}) \leq \var_A({\IS(\pieval, \mathcal{D}, \pib)})\]
where $\var_A$ denotes the asymptotic variance.
\end{corollary}
Corollary 1 states that the asymptotic variance of RIS($n$) must be at least as low as that of $\OIS$.

We first present Theorem 1 from Henmi et al.\ \yrcite{henmi2007importance} and adopt their notation for its presentation.
Consider estimating $v = \Exp{p}{f(x)}$ for probability mass function $p$ and real-valued function $f$.
Given parameterized and twice differentiable probability mass function $q(\cdot | \tilde{\btheta})$, we define the ordinary importance sampling estimator of $v$ as $\tilde{v} = \frac{1}{m} \sum_{i=1}^m \frac{p(x_i)}{q(x_i, \tilde{\btheta})}f(x_i)$.
Similarly, define $\hat{v} = \frac{1}{m} \sum_{i=1}^m \frac{p(x_i)}{q(x_i, \hat{\btheta})}f(x_i)$ where $\hat{\btheta}$ is the maximum likelihood estimate of $\tilde{\btheta}$ given the $m$ samples from $q(\cdot | \tilde{\btheta})$.
The following theorem relates the asymptotic variance of $\hat{v}$ to that of $\tilde{v}$. 
\begin{thm}\label{thm:henmi}
\[ \var_A({\hat{v}}) \leq \var_A(\tilde{v}) \]
where $\var_A$ denotes the asymptotic variance.
\end{thm}
%
\begin{proof}
See Theorem 1 of Henmi et al.\ \yrcite{henmi2007importance}.
\end{proof}

Theorem \ref{thm:henmi} shows that the maximum likelihood estimated parameters of the sampling distribution yield an asymptotically lower variance estimate than using the true parameters, $\tilde{\btheta}$.
To specialize this theorem to our setting, we show that the maximum likelihood behavior policy parameters are also the maximum likelihood parameters for the trajectory distribution of the behavior policy.
First specify the class of sampling distribution: $\Pr(h; \btheta) = p(h) w_\btheta(h)$ where $p(h) = d_0(s_0) \prod_{t=1}^{L-1} P(s_t | s_{t-1}, a_{t-1})$ and $w_\btheta(h) = \prod_{t=0}^{L-1} \pitheta(a_t | s_{t-n}, a_{t-n},\dots,s_t)$.
We now present the following lemma:
\begin{lemma}\label{lemma:asym}
\begin{align*}
&\argmax_\btheta \sum_{h \in \mathcal{D}} \sum_{t=0}^{L-1} \log \pitheta(a_t | s_{t-n}, a_{t-n},\dots,s_t)  = \argmax_\btheta \sum_{h \in \mathcal{D}} \log \Pr(h ; \btheta)
\end{align*}
\end{lemma}
\begin{proof}
\begin{align*}
& \argmax_\btheta \sum_{h \in \mathcal{D}} \sum_{t=0}^{L-1} \log \pitheta(a_t | s_{t-n}, a_{t-n},\dots,s_t) \\
=& \argmax_\btheta \sum_{h \in \mathcal{D}} \sum_{t=0}^{L-1} \log \pitheta(a_t | s_{t-n}, a_{t-n},\dots,s_t) + \underbrace{\log d(s_0) + \sum_{t=1}^{L-1} \log P(s_t | s_{t-1}, a_{t-1})}_\text{const w.r.t. $\btheta$} \\
=& \argmax_\btheta \sum_{h \in \mathcal{D}} \log w_\btheta(h) + \log p(h) \\
 \btheta =& \argmax_\btheta \sum_{h \in \mathcal{D}} \log \Pr(h; \theta)
\end{align*}
\end{proof}


Finally, we combine Lemma \ref{lemma:asym} with Theorem 1 to prove Corollary 1:
\setcounter{corollary}{0}

\begin{corollary}
Let $\Pi_\btheta^n$ be a class of policies, $\pitheta(\cdot | s_{t-n}, a_{t-n},\dots, s_t)$ that are twice differentiable with respect to $\btheta$. If $\exists \btheta \in \Pi_\btheta^n$ such that $\pitheta = \pib$ 
 then \[\var_A({\RIS(n)(\pieval, \mathcal{D})}) \leq \var_A({\IS(\pieval, \mathcal{D}, \pib)})\]
where $\var_A$ denotes the asymptotic variance.
\end{corollary}
\begin{proof}
Define $f(h) = g(h)$, $p(h) = \Pr(h | \pieval)$ and $q(h | \btheta) = \Pr(h | \pitheta)$. Lemma \ref{lemma:asym} implies that:
 \[
\displaystyle\hat{\btheta} = \argmax_{\btheta \in \Pi_\btheta} \sum_{h \in \mathcal{D}} \sum_{t=0}^L \log \pitheta(a_t | s_t)
\]
 is the maximum likelihood estimate of $\tilde{\btheta}$ (where $\pi_{\tilde{\btheta}} = \pib$ and $\Pr(h|\tilde{\btheta})$ is the probability of $h$ under $\pib$) and then Corollary 1 follows directly from Theorem \ref{thm:henmi}.
\end{proof}

Note that for RIS(n) with $n > 0$, the condition that $\pi_{\tilde{\btheta}} \in \Pi^n$ can hold even if the distribution of $A_t \sim \pi_{\tilde{\btheta}}$ (i.e., $A_t \sim \pib$) is only conditioned on $s_t$.
This condition holds when $\exists \pitheta \in \Pi^n$ such that $\forall s_{t-n}, a_{t-n}, \dots a_{t-1}$:
\[ 
 \pi_{\tilde{\btheta}}(a_t | s_t) = \pitheta(a_t | s_{t-n}, a_{t-n},\dots, s_t),
\]
i.e., the action probabilities only vary with respect to $s_t$.

\section{Connection to the REG estimator}

In this appendix we show that $\RIS(L-1)$ is an approximation of the REG 
estimator studied by Li et al.\ \yrcite{li2015toward}.
This connection is notable because Li et al.\ showed $\REG$ is asymptotically minimax optimal, however, in MDPs, $\REG$ requires knowledge of the environment's transition and initial state distribution probabilities while $RIS(L-1)$ does not.
For this discussion, we recall the definition of the probability of a trajectory
 for a given MDP and policy:
\begin{gather*}
\Pr(h | \pi) = d_0(s_0) \pi(a_0 | s_0) P(s_1|s_0, a_0) \cdots P(s_{L-1} | s_{L-2}, a_{L-2})\pi(a_{L-1} | s_{L-1}).
\end{gather*}
We also define $\mathcal{H}$ to be the set of all state-action 
trajectories possible under $\pib$ of length $L$: $s_0, a_0, ... s_{L-1}, a_{L-1}$.

Li et al.\ introduce the regression estimator (REG) for multi-armed bandit problems \yrcite{li2015toward}.
This method estimates the mean reward for each action as $\hat{r}(a, \mathcal{D})$ and then 
computes the REG estimate as: 
\[
\mathtt{REG}(\pieval, \mathcal{D}) = \sum_{a \in \mathcal{A}} \pieval(a) \hat{r}(a, \mathcal{D}).
\]
This estimator is identical to RIS(0) in multi-armed bandit problems \cite{li2015toward}.
The extension of REG to finite horizon MDPs estimates the mean 
return for each trajectory as $\hat{g}(h, \mathcal{D})$ and then computes the estimate:
\[
\mathtt{REG}(\pieval, \mathcal{D}) = \sum_{h \in \mathcal{H}} \Pr(h | \pieval) \hat{g}(h, \mathcal{D}).
\]
Since this estimate uses $\Pr(h | \pieval)$ it requires knowledge of the initial state distribution, $d_0$, and transition probabilities, $P$.
%

We now elucidate a relationship between $\RIS(L-1)$ and REG even though they are different estimators.
Let $c(h)$ denote the number of times that trajectory $h$ appears in $\mathcal{D}$.
We can rewrite $\REG$ as an importance sampling method with a count-based estimate of the probability of a trajectory in the denominator:
\begin{align}
\mathtt{REG}(\pieval, \mathcal{D}) =& \sum_{h \in \mathcal{H}} \Pr(h | \pieval) \hat{g}(h, \mathcal{D}) \\ =& \frac{1}{m} \sum_{h \in \mathcal{H}} c(h) \frac{\Pr(h | \pieval)}{c(h) / m} \hat{g}(h, \mathcal{D}) \\
                                                                              =& \frac{1}{m} \sum_{i=1}^m\frac{\Pr(h_i | \pieval)}{c(h_i) / m} g(h_i) \label{eq1}
\end{align}

The denominator in (\ref{eq1}) can be re-written as a telescoping product to obtain 
an estimator that is similar to RIS($L-1$):
\begin{align*}
\mathtt{REG}(\pieval, \mathcal{D}) =& \frac{1}{m} \sum_{i=1}^m\frac{\Pr(h_i | \pieval)}{c(h_i) / m}g(h_i) \\ =&
 \frac{1}{m} \sum_{i=1}^m\frac{\Pr(h_i | \pieval)}{\frac{c(s_0)}{m}\frac{c(s_0, a_0)}{c(s_0)}\cdots\frac{c(h_i)}{c(h_i / a_{L-1})}}g(h_i) \\=&
  \frac{1}{m} \sum_{i=1}^m\frac{d_0(s_0) \pieval(a_0 | s_0) P(s_1|s_0, a_0) \cdots P(s_{L-1} | s_{L-2}, a_{L-2})\pieval(a_{L-1} | s_{L-1})}{\hat{d}(s_0)\pidata(a_0|s_0)\hat{P}(s_1 | s_0, a_0) \cdots \hat{P}(s_{L-1} | h_{0:L-1}) \pidata(a_{L-1} | h_{i:j}) }g(h_i).
\end{align*}
This expression differs from RIS($L-1$) in two ways:
\begin{enumerate}
\item The numerator includes the initial state distribution and 
transition probabilities of the environment.
\item The denominator includes count-based estimates of the initial state distribution and transition probabilities of the environment where the transition probabilities are conditioned on all past states and actions.
\end{enumerate}
If we assume that the empirical estimates of the environment probabilities in 
the denominator are equal to the true environment probabilities then these 
factors cancel and we obtain the RIS($L-1$) estimate.
This assumption will almost always be false except in deterministic environments.
However, showing that RIS($L-1$) is approximating REG suggests that $\RIS(L-1)$
may have similar theoretical properties to those elucidated for REG by Li et al.\ \yrcite{li2015toward}.
Our SinglePath experiment (See Figure 2 in the main text) supports this conjecture: RIS($L-1$) has high bias in the low to medium sample size but have asymptotically lower $\mse$ compared to other methods.
$\REG$ has even higher bias in the low to medium sample size range but has asymptotically lower $\mse$ compared to $\RIS(L-1)$.
$\RIS$ with smaller $n$ appear to decrease the initial bias but have larger $\mse$ as the sample size grows.
The asymptotic benefit of $\RIS$ for all $n$ is also corroborated by Corollary 1 in Appendix B though Corollary 1 does \textit{not} tell us anything about how different $\RIS$ methods compare asymptotically.
The asymptotic benefit of $\REG$ compared to $\RIS$ methods can be understood as $\REG$ correcting for sampling error in both the action selection and state transitions.

\section{Sampling Error with Continuous Actions}

In Section 3 of the main text we discussed how ordinary importance sampling can suffer from sampling error.
Then, in Section 4, we presented an example showing how $\RIS$ corrects for sampling error in $\mathcal{D}$ in deterministic and finite MDPs.
Most of this discussion assumed that the state and action spaces of the MDP were finite.
Here, we discuss sampling error in continuous action spaces.
%
The primary purpose of this discussion is intuition and we limit discussion to a setting that can be easily visualized.
We consider a deterministic MDP with scalar, real-valued actions, reward $R: \mathcal{A} \rightarrow \mathbb{R}$, and $L=1$.

We assume the support of $\pib$ and $\pieval$ is bounded and for simplicity assume the support to be $[0,1]$.
Policy evaluation is equivalent to estimating the integral: 
\begin{equation}\label{eq:vpi}
v(\pieval) = \int_0^1 R(a) \pieval(a) da
\end{equation}
and the ordinary importance sampling estimate of this quantity with $m$ samples from $\pib$ is:
\begin{equation}\label{eq:mc-vpi}
\frac{1}{m} \sum_{i=1}^m \frac{\pieval(a_i)}{\pib(a_i)} R(a_i).
\end{equation}

Even though the $\OIS$ estimate is a sum over a finite number of samples, we show it is exactly equal to an integral over a particular piece-wise function.
We assume (w.l.o.g) that the $a_i$'s are in non-decreasing order, ($a_0 <= a_i <= a_m$).
Imagine that we place the $R(a_i)$ values uniformly across the interval $[0,1]$ so that they divide the range $[0,1]$ into $m$ equal bins.
In other words, we maintain the relative ordering of the action samples but ignore the spatial relationship between samples.
We now define piece-wise constant function $\bar{R}_{\OIS}$ where $\bar{R}_{\OIS}(a) = R(a_i)$ if $a$ is in the $i^\text{th}$ bin.
The ordinary importance sampling estimate is exactly equal to the integral $\int_0^1 \bar{R}_{\OIS}(a) da$.

\begin{figure*}[ht!]
\subfigure[Policy and Reward]{\includegraphics[width=0.33\linewidth]{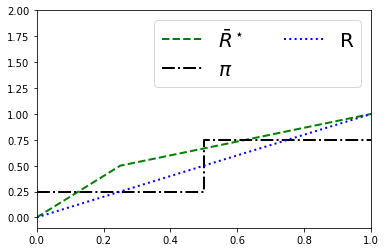}\label{fig:continuous:setup}}
\subfigure[10 Sample Approximation]{\includegraphics[width=0.33\linewidth]{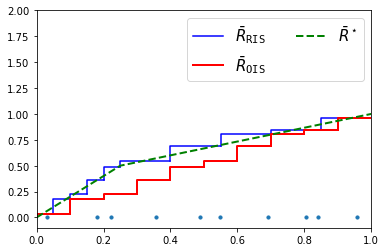}\label{fig:continuous:10}}
\subfigure[200 Sample Approximation]{\includegraphics[width=0.33\linewidth]{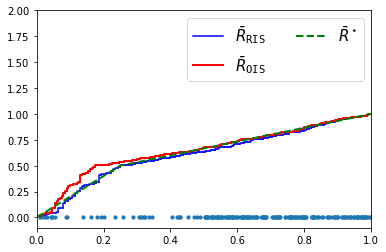}\label{fig:continuous:200}}
\caption{Policy evaluation in a continuous armed bandit task. Figure \ref{fig:continuous:setup} shows a reward function, $R$, and the pdf of a policy, $\pi$, with support on the range $[0,1]$. With probability $0.25$, $\pi$ selects an action less than $0.5$ with uniform probability; otherwise $\pi$ selects an action greater than $0.5$. The reward is equal to the action chosen.
All figures show $\bar{R}^\star$: a version of $R$ that is stretched according to the density of $\pi$; since the range $[0.5,1]$ has probability $0.75$, $R$ on this interval is stretched over $[0.25, 1]$.
 Figure \ref{fig:continuous:10} and \ref{fig:continuous:200} show $\bar{R}^\star$ and the piece-wise $\bar{R}_{\OIS}$ and $\bar{R}_{\RIS}$ approximations to $\bar{R}^\star$ after $10$ and $200$ samples respectively.}
\label{fig:continuous}
\end{figure*}

It would be reasonable to assume that $\bar{R}_{\OIS}(a)$ is approximating $R(a) \pieval(a)$ since the ordinary importance sampling estimate (\ref{eq:mc-vpi}) is approximating (\ref{eq:vpi}), i.e., $\displaystyle \lim_{m\rightarrow \infty} \bar{R}_{\OIS}(a) = R(a) \pieval(a)$.
In reality, $\bar{R}_{\OIS}$ approaches a \textit{stretched} version of $R$ where areas with high density under $\pieval$ are stretched and areas with low density are contracted.
We call this stretched version of $R$, $\bar{R}^\star$.
The integral of $\int_0^1 \bar{R}^\star(a) da$ is $v(\pieval)$.

Figure \ref{fig:continuous:setup} gives a visualization of an example $\bar{R}^\star$ using on-policy Monte Carlo sampling from an example $\pieval$ and linear $R$.
In contrast to the true $\bar{R}^\star$, the $\OIS$ approximation to $\bar{R}$, $\bar{R}_{\OIS}$ stretches ranges of $R$ according to the number of samples in that range: ranges with many samples are stretched and ranges without many samples are contracted.
As the sample size grows, any range of $R$ will be stretched in proportion to the probability of getting a sample in that range.
For example, if the probability of drawing a sample from $[a,b]$ is $0.5$ then $\bar{R}^\star$ stretches $R$ on $[a,b]$ to cover half the range $[0,1]$.
Figure \ref{fig:continuous} visualizes $\bar{R}_{\OIS}$ the $\OIS$ approximation to $\bar{R}^\star$ for sample sizes of $10$ and $200$.

In this analysis, sampling error corresponds to over-stretching or under-stretching $R$ in any given range.
The limitation of ordinary importance sampling can then be expressed as follows: given $\pieval$, we know the correct amount of stretching for any range and yet $\OIS$ ignores this information and stretches based on the empirical proportion of samples in a particular range.
On the other hand, $\RIS$ first divides by the empirical pdf (approximately undoing the stretching from sampling) and then multiplies by the true pdf to stretch $R$ a more accurate amount.
Figure \ref{fig:continuous} also visualizes the $\bar{R}_{\RIS}$ approximation to $\bar{R}^\star$ for sample sizes of $10$ and $200$.
In this figure, we can see that $\bar{R}_{\RIS}$ is a closer approximation to $\bar{R}^\star$ than $\bar{R}_{\OIS}$ for both sample sizes.
In both instances, the mean squared error of the $\RIS$ estimate is less than that of the $\OIS$ estimate.

Since $R$ may be unknown until sampled, we will still have non-zero $\mse$.
However the standard $\OIS$ estimate has error due to \textit{both} sampling error and unknown $R$ values.

\section{Extended Empirical Description}

In this appendix we provide additional details for our experimental domains. Code is provided at \url{https://github.com/LARG/regression-importance-sampling}.

\paragraph{SinglePath:}
This environment is shown in Figure \ref{fig:singlepath-appendix} with horizon $L=5$.
In each state, $\pib$ selects action, $a_0$, with probability $p=0.6$ and $\pieval$ selects action, $a_0$, with probability $1 - p=0.4$.
Action $a_0$ causes a deterministic transition to the next state. Action $a_1$ causes a transition to the next state with probability $0.5$, otherwise, the agent remains in its current state.
The agent receives a reward of $1$ for action $a_0$ and $0$ otherwise.
RIS uses count-based estimation of $\pib$ and REG uses count-based estimation of trajectories. REG is also given the environment's transition matrix, $P$.

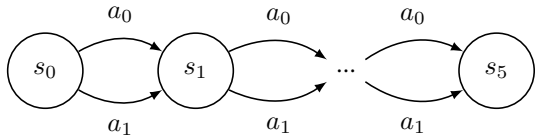
\begin{figure}
\centering
\begin{tikzpicture}[-latex ,auto ,node distance =1.75 cm and 2.75cm ,
semithick ,
state/.style ={ circle ,top color =white , bottom color = white,
draw,black , text=black , minimum width =1 cm}, inf/.style = {}]
\node[state] (start) at (0.0, 0.0){$s_0$};
\node[state] (next) at (2.0, 0.0) {$s_1$};
\node[inf] (dots) at (4.0, 0.0) {...};
\node[state] (end) at (6.0, 0.0) {$s_5$};
\path (start) edge [bend left] node[above=0.1 cm,align=left] {$a_0$} (next);
\path (start) edge [bend right] node[below=0.1 cm,align=left] {$a_1$} (next);
\path (next) edge [bend left] node[above=0.1 cm,align=left] {$a_0$} (dots);
\path (next) edge [bend right] node[below=0.1 cm,align=left] {$a_1$} (dots);
\path (dots) edge [bend left] node[above=0.1 cm,align=left] {$a_0$} (end);
\path (dots) edge [bend right] node[below=0.1 cm,align=left] {$a_1$} (end);
\end{tikzpicture}
\caption{The SinglePath MDP referenced in Section 4 of the main text. \textbf{Not shown:} If the agent takes action $a_1$ it remains in its current state with probability $0.5$.}
\label{fig:singlepath-appendix}
\end{figure}

\paragraph{Gridworld:}

This domain is a $4 \times 4$ Gridworld with a terminal state with reward $100$ at $(3, 3)$, a state with reward $-10$ at $(1,1)$, a state with reward $1$ at $(1, 3)$, and all other states having reward $-1$.
The domain has been used in prior off-policy policy evaluation work \cite{thomas2015safe,thomas2016data-efficient,hanna2017data-efficient, farajtabar2018more}.
The action set contains the four cardinal directions and actions move the agent in its intended direction (except when moving into a wall which produces no movement).
The agent begins in $(0,0)$, $\gamma=1$, and $L=100$.
All policies use a softmax action selection distribution with temperature $1$ and a separate parameter, $\theta_{sa}$, for each state, $s$, and action $a$. 
The probability of taking action $a$ in state $s$ is given by:
\[ \pi(a|s) = \frac{e^{\theta_{sa}}}{\sum_{a^\prime \in \mathcal{A}} e^{\theta_{sa^\prime}}}\]
The first set of experiments uses a behavior policy, $\pib$, that can reach the high reward terminal state and an evaluation policy, $\pieval$, that is the same policy with lower entropy action selection.
The second set of experiments uses the same behavior policy as both behavior and evaluation policy.
$\RIS$ estimates the behavior policy with the empirical frequency of actions in each state.
This domain allows us to study $\RIS$ separately from questions of function approximation.

\paragraph{Linear Dynamical System}

This domain is a point-mass agent moving towards a goal in a two dimensional world by setting $x$ and $y$ acceleration.
The state-space is the agent's $x$ and $y$ position and velocity.
The agent acts for $L=20$ time-steps under linear-gaussian dynamics and receives a reward that is proportional to its distance from the goal.
Specifically, if $\mathbf{s}_t$ is the agent's state vector and it takes action $\mathbf{a}_t$, then the resulting next state is:
\[
\mathbf{s}_{t+1} = A \cdot \mathbf{s}_t + B \cdot \mathbf{a}_t + \mathbf{\epsilon}_t
\]
where $\mathbf{\epsilon}_t \sim \mathcal{N}(0, I)$, $A$ is the identity matrix, and 
\[
B = \begin{bmatrix}
    0.5  & 0   \\
    0    & 0.5 \\
    1    & 0   \\
    0    & 1
\end{bmatrix}.
\]

The agent's policy is a linear map from state features to the mean of a Gaussian distribution over actions.
For the state features, we use second order polynomial basis functions so that policies are non-linear in the state features but we can still estimate $\pidata$ efficiently with ordinary least squares.
We obtain a basic policy by optimizing the linear weights of this policy for 10 iterations of the Cross-Entropy method \cite{rubinstein2013cross}.
The evaluation policy uses a standard deviation of $0.5$ and the true $\pib$ uses a standard deviation of $0.6$.

\paragraph{Continuous Control}
We also use two continuous control tasks from the OpenAI gym: Hopper and HalfCheetah.\footnote{For these tasks we use the Roboschool versions: \url{https://github.com/openai/roboschool}}
The state and action dimensions of each task are shown in Table \ref{tab:dimensions}.
\begin{table}[h!]
\centering
\begin{tabular}{c|c|c}
Environment  & State Dimension & Action Dimension \\ \hline
Hopper       & 15              & 3                \\
Half Cheetah & 26              & 6                \\
\end{tabular}
\caption{State and action dimension for each OpenAI Roboschool environment.}
\label{tab:dimensions}
\end{table}
In each task, we use neural network policies with $2$ layers of $64$ hidden units each for $\pieval$ and $\pib$.
Each policy maps the state to the mean of a Gaussian distribution with state-independent standard deviation.
We obtain $\pieval$ and $\pib$ by running the OpenAI Baselines \cite{baselines} version of proximal policy optimization (PPO) \cite{schulman2017proximal} and then selecting two policies along the learning curve. 
For both environments, we use the policy after $30$ updates for $\pieval$ and after $20$ updates for $\pib$.
These policies use $\tanh$ activations on their hidden units since these are the default in the OpenAI Baselines PPO implementation.

$\RIS$ estimates the behavior policy with gradient descent on the negative log-likelihood of the neural network.
Specifically, we interpret the neural network outputs, $\mu(s)$, as the mean of a multi-variate Gaussian distribution with diagonal covariance matrix.
We use a state-independent parameter vector, $\sigma$, to represent the log-standard deviation of the Gaussian distribution.
Given $m$, state-action pairs, $\RIS$ uses the loss function:
\[
\loss = \sum_{i=1}^m  0.5 ((a_i - \mu(s_i)) / e^\sigma) ^ 2 + \sigma
\]
Minimizing $\loss$ is equivalent to minimizing a squared-error loss function with regards to estimating $\mu$.

In our experiments we use a learning rate of $1\times 10^{-3}$ and L2-regularization with a weight of $0.02$. The multi-layer behavior policies learned by $\RIS$ use relu activations.
The specific architectures considered for $\pidata$ have either $0$, $1$, $2$, or $3$ hidden layers with $64$ units in each hidden layer.

In these domains we only consider a batch size of $400$ trajectories for estimating $\pidata$ and computing the policy value estimate.
For determining early stopping and measuring validation error we use a separate batch of $80$ trajectories (20\% of the policy evaluation data).

\section{Extended Empirical Results}

This appendix includes two additional plots that space constraints limited from the main text.

\subsection{Importance Sampling Variants}

This appendix presents additional importance sampling methods that are implemented with both $\OIS$ weights and $\RIS$ weights.
Specifically, we implement the following:
\begin{itemize}
\item The ordinary importance sampling estimator described in Section 2.
\item The weighted importance sampling estimator ($\WIS$) \cite{precup2000eligibility} that normalizes the importance weights with their sum.
\item Per-decision importance sampling ($\PDIS$) \cite{precup2000eligibility} that importance samples the individual rewards.
\item The doubly-robust ($\DR$) estimator \cite{jiang2016doubly,thomas2016data-efficient} that uses a model of $P$ and $r$ to lower the variance of $\PDIS$.
\item The weighted doubly robust ($\WDR$) estimator \cite{thomas2016data-efficient} that uses weighted importance sampling to lower the variance of the doubly robust estimator.
\end{itemize}
Since $\DR$ and $\WDR$ require a model of the environment, we estimate a count-based model with half of the available data in $\mathcal{D}$.

Figure \ref{fig:gridworld-appendix} gives results for all $5$ of these $\IS$ variants implemented with both $\RIS$ weights and $\OIS$ weights.
Figure \ref{fig:gridworld-onpol-appendix} gives the same results except for the on-policy setting. Note that in the on-policy setting, $\PDIS$ and $\WIS$ are identical to $\IS$ and $\WDR$ is identical to $\DR$ when implemented with $\OIS$ weights. Thus we only present the $\RIS$ versions of these methods.
In addition to the results for ordinary $\IS$, $\WIS$, and $\WDR$ that are also in the main text, Figure \ref{fig:gw-appendix} shows $\RIS$ weights improve $\DR$ and $\PDIS$.

\begin{figure}[h!]
\centering
\subfigure[Gridworld Off-Policy]{\includegraphics[width=0.45\linewidth]{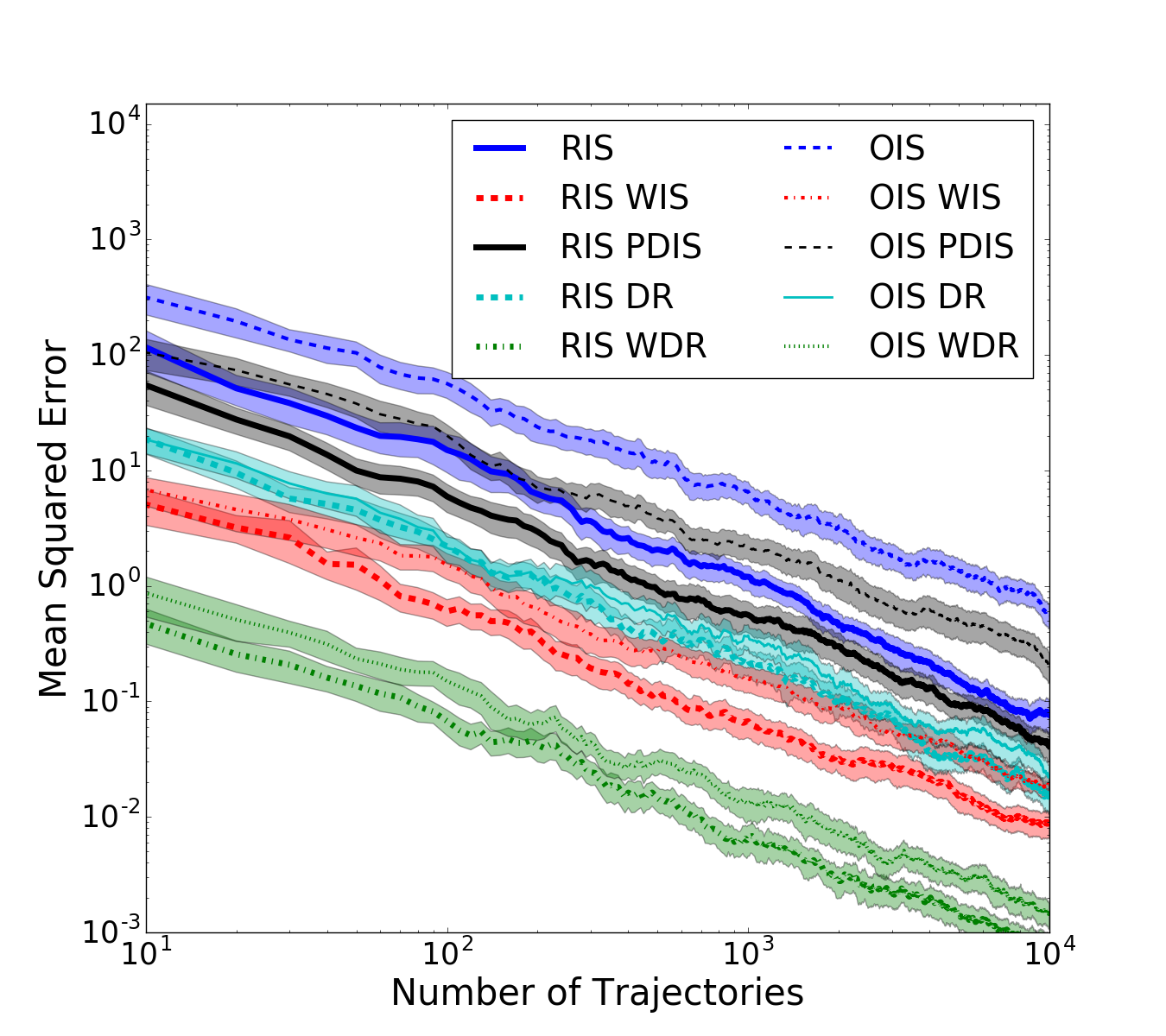}\label{fig:gridworld-appendix}}
\subfigure[Gridworld On-Policy]{\includegraphics[width=0.45\linewidth]{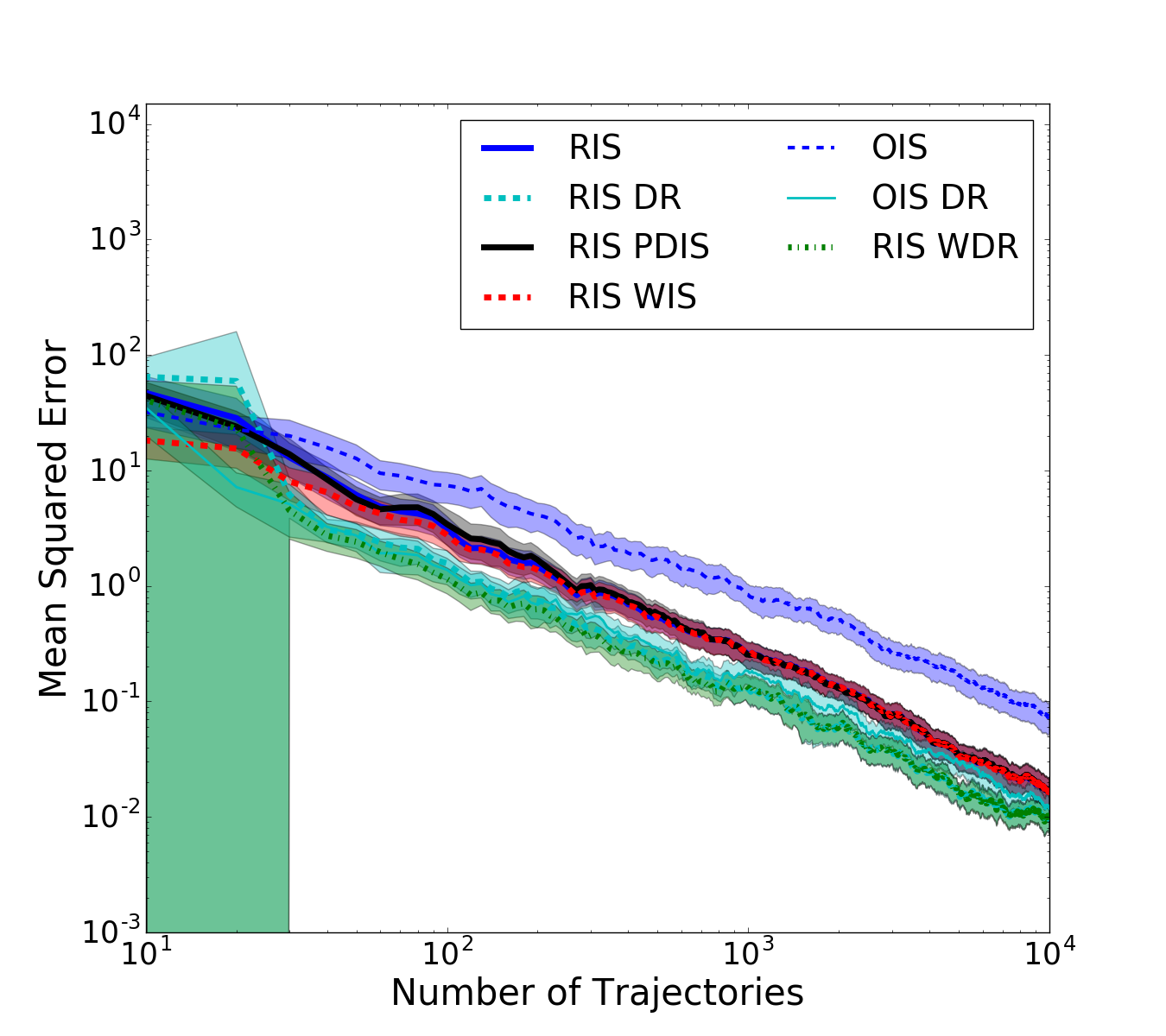}\label{fig:gridworld-onpol-appendix}}
\caption{Policy evaluation results for Gridworld. In all subfigures, the x-axis is the number of trajectories collected and the y-axis is mean squared error. Axes are log-scaled. The shaded region gives a 95\% confidence interval. The main point of comparison is the $\RIS$ variant of each method to the $\OIS$ variant of each method, e.g., $\RIS$ $\WIS$ compared to $\OIS$ $\WIS$. Results are averaged over $100$ trials.\label{fig:gw-appendix}}
\end{figure}

\subsection{Gradient Descent Policy Estimation}

This appendix shows how the $\mse$ of $\RIS$ changes during estimation of $\pidata$ in the HalfCheetah domain.
Figure \ref{fig:fit_exp-cheetah} gives the results.
As in the Hopper domain, we see that the minimal validation loss policy and the minimal $\mse$ policy are misaligned.
The $\RIS$ estimate initially over-estimates the policy value and then begins under-estimating.
Further discussion of these observations are given in Section 6 of the main text.

\begin{figure}
\centering
\includegraphics[width=0.65\linewidth]{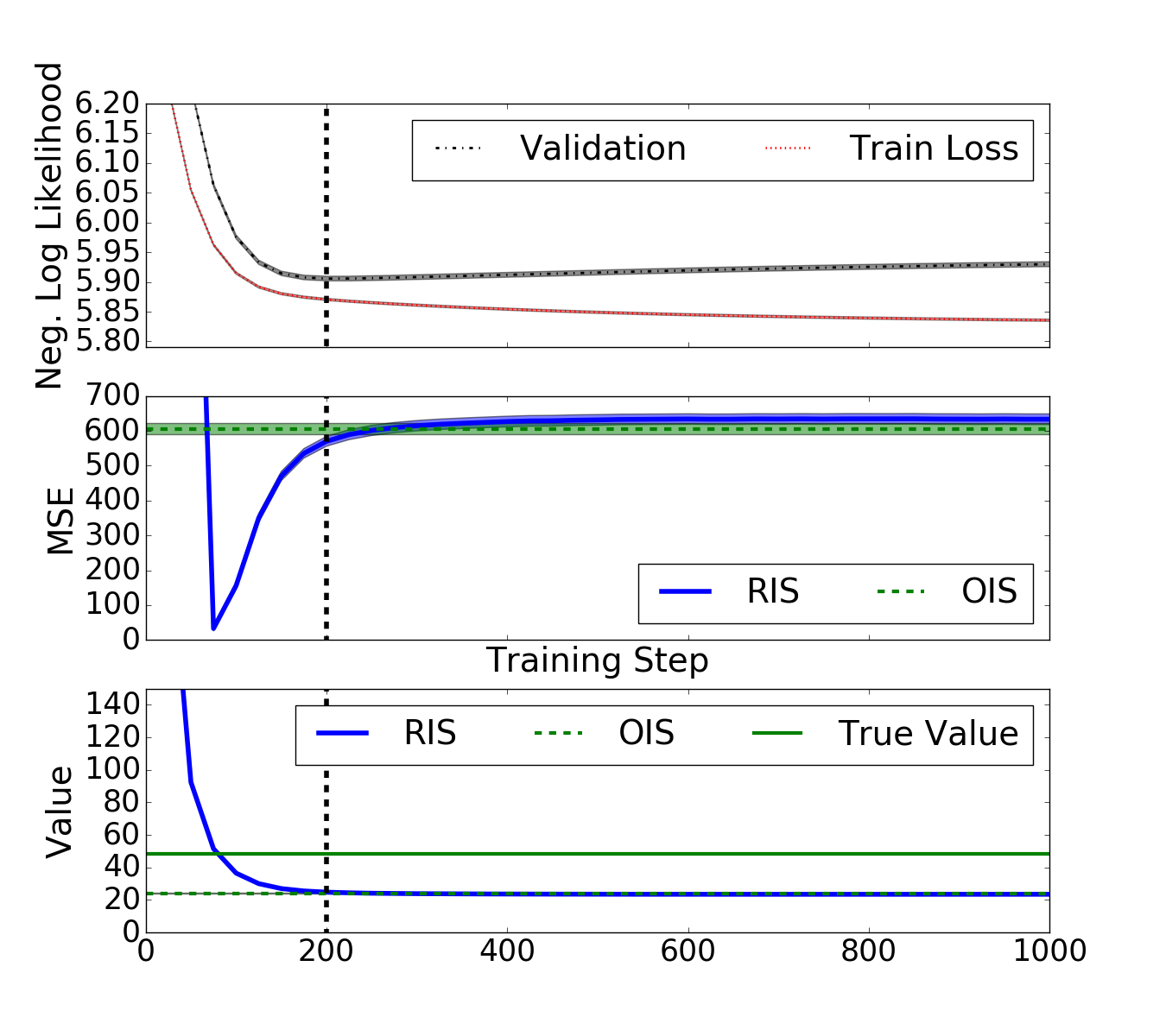}
\caption{Mean squared error and estimate of the importance sampling estimator during training of $\pidata$. The x-axis is the number of gradient ascent steps. The top plot shows the training and validation loss curves. The y-axis of the top plot is the average negative log-likelihood. 
The y-axis of the middle plot is mean squared error (MSE).
The y-axis of the bottom plot is the value of the estimate.
MSE is minimized close to, but slightly before, the point where the validation and training loss curves indicate that overfitting is beginning.
This point corresponds to where the $\RIS$ estimate transitions from over-estimating to under-estimating the policy value.
Results are averaged over $200$ trials and the shaded region represents a 95\% confidence interval around the mean result.
 \label{fig:fit_exp-cheetah}}
\end{figure}

\end{appendix}

\end{document}